\documentclass[journal]{IEEEtran}
\usepackage[utf8]{inputenc}
\usepackage{comment}
\usepackage{graphicx}
\usepackage{xcolor}
\usepackage{wrapfig}
\usepackage{lipsum}
\definecolor{wpired}{RGB}{172, 43, 55}
\definecolor{rred}{RGB}{169, 50, 38}
\definecolor{ggreen}{RGB}{34, 153, 84}
\definecolor{bblue}{RGB}{36, 113, 163}
\definecolor{ppurple}{RGB}{125, 60, 152}
\definecolor{yyellow}{RGB}{214, 137, 16}
\definecolor{ggrey}{RGB}{112, 123, 124}

\usepackage{todonotes}

\usepackage{mathtools}
\usepackage{amsmath}
\usepackage{amssymb}
\usepackage{amsthm}
\theoremstyle{definition}
\usepackage{tikz}
\usetikzlibrary{automata,positioning,angles,quotes}
\usepackage{tikz-cd}
\usepackage[cal=cm, scr=dutchcal]{mathalfa}

\newtheoremstyle{main}
{0.5em} 
{0.5em} 
{\normalfont} 
{0pt} 
{\bfseries} 
{} 
{2pt} 
{\thmname{#1}\thmnumber{ #2. }\thmnote{\normalfont\itshape (#3)\normalfont:\\}}
\newtheoremstyle{math}
{0.5em} 
{0.5em} 
{\itshape} 
{0pt} 
{\bfseries} 
{} 
{2pt} 
{\thmname{#1}\thmnumber{ #2. }\thmnote{\normalfont\itshape (#3)\normalfont:\\}}
\theoremstyle{main}
\newtheorem{assumption}{Assumption} 
\newtheorem{definition}{Definition} 

\theoremstyle{math}
\newtheorem{theorem}{Theorem} 
\newtheorem{corollary}{Corollary}
\newtheorem{lemma}{Lemma} 
\newtheorem{proposition}{Proposition} 

\makeatletter
\let\NAT@parse\undefined
\makeatother
\usepackage{cite}
\usepackage[pdfa,colorlinks,bookmarks=true,bookmarksopen,bookmarksnumbered,allcolors=ggreen]{hyperref}

\newcommand{\appref}[1]{\hyperref[#1]{Appendix~\ref*{#1}}}

\usepackage[algoruled,vlined,linesnumbered]{algorithm2e}

\SetAlgoSkip{SkipBeforeAndAfter}

\SetCommentSty{mycommfont}
\SetKw{Continue}{continue}
\SetAlgorithmName{Algorithm}{Alg.}

\usepackage{booktabs}
\usepackage{multirow}
\usepackage{multicol}
\usepackage{makecell}

\usepackage[font=footnotesize]{caption}
\usepackage[font=footnotesize]{subcaption}
\usepackage[export]{adjustbox}

\usepackage[
activate   = {true},
protrusion = true,
expansion  = true,
kerning    = true,
spacing    = true,
tracking   = false,
auto       = true,
selected   = true,
factor     = 1000,
stretch    = 25,
shrink     = 25,
]{microtype}

\usepackage{xspace}

\newcommand{\method}{\text{KiTe}\xspace}

\newcommand{\aug}{z}
\newcommand{\Y}{{\ensuremath{\mathcal{Y}}}\xspace}
\newcommand{\heuristics}{\mathcal{H}}

\newcommand{\Rplus}{\mathbb{R}_+}

\newcommand{\cost}{\mathcal{C}}
\newcommand{\length}{\ell}
\newcommand{\lengthb}{\mathscr{w}}
\newcommand{\terminal}{\phi}
\newcommand{\terminalb}{\psi}
\newcommand{\Yb}{{\ensuremath{\mathcal{Y}_b}}\xspace}
\newcommand{\cov}{P}

\newcommand{\ball}{B}

\newcommand{\vol}[1]{\mathcal{V}\!\left[#1\right]}

\newcommand{\Exp}{\mathrm{Exp}}
\newcommand{\Log}{\mathrm{Log}}
\newcommand{\Adj}{\mathrm{Ad}}

\newcommand{\start}{\text{start}}
\newcommand{\goal}{\text{goal}}
\newcommand{\free}{\text{free}}

\newcommand{\SEThree}{\ensuremath{SE(3)}\xspace}
\newcommand{\SETwo}{\ensuremath{SE(2)}\xspace}

\newcommand{\SPD}{{\ensuremath{\mathbb{S}}}\xspace}

\newcommand{\X}{\ensuremath{\mathcal{X}}\xspace}
\newcommand{\G}{\ensuremath{\mathcal{G}}\xspace}

\newcommand{\B}{\ensuremath{\mathcal{B}}\xspace}
\newcommand{\U}{\ensuremath{\mathcal{U}}\xspace}
\newcommand{\Xfree}{\ensuremath{\X_{\text{free}}}\xspace}
\newcommand{\Xobs}{\ensuremath{\X_{\text{obs}}}\xspace}

\newcommand{\xtraj}{\boldsymbol{x}}
\newcommand{\btraj}{\boldsymbol{b}}
\newcommand{\utraj}{\boldsymbol{u}}
\newcommand{\ytraj}{\boldsymbol{y}}

\newcommand{\T}{{\ensuremath{\mathcal{T}}}\xspace}

\title{
Terminal Matters: Kinodynamic Planning with \\
a Terminal Cost and Learned Uncertainty \\
in Belief State-Cost Space
}

\author{Zhuoyun Zhong, Seyedali Golestaneh, and Constantinos Chamzas
  \thanks{%
   All authors are affiliated with the Department of Robotics Engineering, Worcester Polytechnic Institute (WPI), Worcester, MA 01609, USA {\tt\small \{zzhong3, sgolestaneh, cchamzas\} @ wpi.edu}.
  }
}

\begin{document}

\maketitle

\begin{abstract}
In many real-world robotic tasks, robots must generate dynamically feasible motions that reliably reach desired goals even under uncertainty.
Yet existing sampling-based kinodynamic planners typically optimize running costs accumulated along trajectories and treat goal reaching as a feasibility check, rather than explicitly optimizing terminal-state quality, such as goal preference or goal-reaching reliability.
In this work, we introduce Kinodynamic planning with a Terminal cost, termed \method.
It augments AO-RRT with a terminal-cost objective to optimize terminal-state quality alongside trajectory running cost.
We provide a rigorous proof that this augmented AO-RRT formulation preserves asymptotic optimality.
Furthermore, we extend the formulation to belief space and prove that minimizing the Wasserstein distance between the terminal belief and the goal improves a lower bound on the probability of reaching the goal region.
To support systems without analytical uncertainty models, we learn dynamics and process uncertainty directly from data and integrate the learned belief dynamics into planning.
We validate the proposed formulation with experiments on Flappy Bird, Car Parking, and Planar Pushing with learned belief dynamics in both simulation and the real world.
Compared with baselines without terminal costs, \method better satisfies goal preferences and consistently achieves higher goal-reaching success under uncertainty.
Source code is available at 
\url{https://github.com/elpis-lab/KiTe}.
%
\end{abstract}
\begin{IEEEkeywords}
Motion and Path Planning, Kinodynamic Planning, Planning Under Uncertainty, Learning Belief Dynamics
\end{IEEEkeywords}

\section{Introduction}
\label{sec:introduction}
Robots operating in the real world must generate motions that satisfy kinodynamic constraints~\cite{lavalle2006planning} while accounting for uncertainty stemming from execution errors \cite{mapping_and_planning, multi_cc}, imperfect models~\cite{state_epistemic, ltamp}, and complex contact interactions~\cite{poking, activepusher}.
Sampling-based planners provide a general framework for kinodynamic systems, with probabilistic completeness~\cite{lavalle2001randomized} and, in asymptotically optimal variants, convergence guarantees\mbox{~\cite{review-asym,sst,aox}}.
In parallel, execution uncertainty has motivated belief-space planning methods, where the system state is represented as a probability distribution and uncertainty is propagated explicitly during planning~\cite{rrbt,ccrrt,stochastic_robustness}.
Although these methods provide principled tools for kinodynamic planning under uncertainty, their objectives typically emphasize accumulated running cost along the planned trajectory.
%

However, many robotic tasks depend not only on the running cost accumulated during motion, such as path length, control effort, or collision risk, but also on the quality of the terminal state.
%
%
In the pushing example in~\autoref{fig:intro}(a), the blue trajectory is geometrically shorter, but its terminal uncertainty is larger and therefore may lead to a lower probability of reaching the goal after execution.
In contrast, the purple trajectory accepts a longer nominal path to reach the goal with higher probability.
A terminal cost can also naturally encode preferences among feasible goals.
In the parking example in~\autoref{fig:intro}(b), the purple trajectory reaches the preferred open-space goal region rather than the closer alternative, while also achieving a more reliable terminal belief state.
These examples highlight that terminal-state quality should be optimized directly in kinodynamic planning, rather than treated only as a feasibility condition.

%
%
Planning in belief space also requires propagating both nominal dynamics and uncertainty~\cite{belief_tree}.
This creates a practical challenge in complex systems, such as contact-rich pushing in~\autoref{fig:intro}~(a), where analytical models are often unavailable.
It therefore motivates methods for learning belief dynamics models that capture state transitions and process uncertainty in a form compatible with belief-space planning.

\begin{figure}
    \centering
    \begin{subfigure}[t]{\linewidth}
        \centering
        \includegraphics[width=\linewidth]{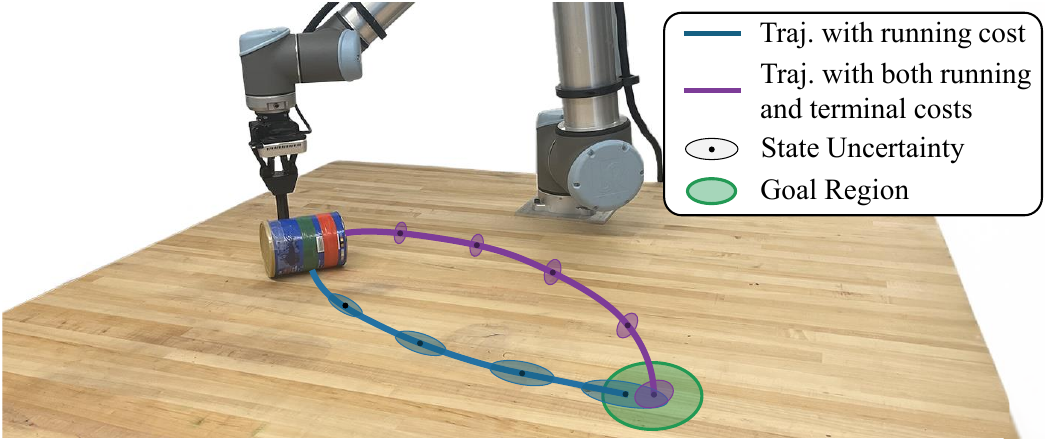}
        \vspace{-4ex}
        \caption{Planar Pushing Task.}
        \label{fig:intro1}
    \end{subfigure}
    \begin{subfigure}[t]{\linewidth}
        \centering
        \includegraphics[width=\linewidth]{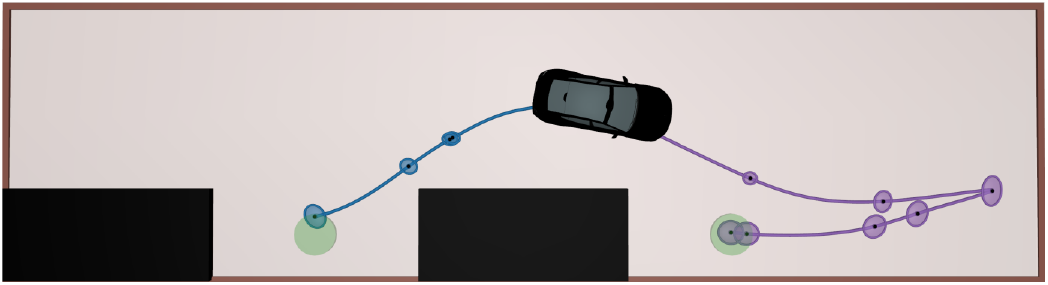}
        \vspace{-4ex}
        \caption{Car Parking Task.}
        \label{fig:intro2}
    \end{subfigure}
    \caption{
    Motivating examples for kinodynamic planning with an explicit terminal cost. 
    (a) In Planar Pushing, the longer purple trajectory pushes from a more stable side of the object, yielding a lower-uncertainty terminal belief and thus a higher probability of reaching the goal.
    (b) In Car Parking, multiple goal regions may be feasible, but not be equally desirable. A terminal cost can encode a preference for the open parking spot while accounting for terminal uncertainty to improve goal-reaching reliability.
    }
    \label{fig:intro}
\vspace{-2ex}
\end{figure}

This paper proposes sampling-based \textbf{Ki}nodynamic planning with a \textbf{Te}rminal cost, termed \method, and highlights the importance of explicit terminal-state optimization.
First, in \autoref{sec:aorrt-terminal}, we show that a terminal-cost objective can be optimized along the running cost within the kinodynamic planner AO-RRT \cite{aox, aorrt} and provide a rigorous proof that AO-RRT preserves asymptotic optimality.
Furthermore, in \autoref{sec:aorrt-belief}, we extend this formulation to belief space and prove that the planner remains asymptotically optimal.

Belief-space planners can explicitly reason about uncertainty, but they often rely on hard probability thresholds or trajectory-wise surrogates instead of directly optimizing the terminal belief~\cite{belief_tree, ccrrt, multi_cc}.
To address this issue, in \autoref{sec:aorrt-bound}, we prove that the probability of reaching a goal region admits a lower bound in terms of the 2-Wasserstein distance between the terminal belief and the goal.
This yields a principled terminal objective for goal-reaching probability.
The same formulation also allows the planner to encode goal preference.

Lastly, in \autoref{sec:modeling}, we show that suitable belief dynamics models can be learned directly from interaction data \cite{poking} using neural networks trained with negative log-likelihood (NLL) loss.
The learned model predicts both the mean transition and the associated process uncertainty, and can be directly  integrated into kinodynamic planning.

In \autoref{sec:experiments}, comprehensive experiments are conducted on three tasks with diverse systems: Flappy Bird, Car Parking, and Planar Pushing. 
The proposed framework is compared against classical and state-of-the-art kinodynamic motion planners, in both simulation and the real world.
The results show that our method can account for goal-state quality while consistently achieving higher goal-reaching success rates than the baselines.

In summary, the core contributions of this paper are:
\begin{itemize}
    \item We implement AO-RRT with a terminal cost and provide a rigorous proof that its asymptotic optimality is preserved. We further prove that it remains asymptotic optimal when planning in belief space.
    \item We prove that the Wasserstein distance between the terminal belief and the goal induces a lower bound on goal-reaching success, yielding a principled terminal objective.
    \item We demonstrate how learned belief dynamics models can explicitly capture both state transitions and process uncertainty for belief-state propagation within belief-space kinodynamic planning.
    \item We validate the proposed method \method across various tasks in simulation and the real world.
\end{itemize}

\section{Related Work}
\label{sec:related}
Motion planning under uncertainty for dynamical systems has been studied through sampling-based kinodynamic planning, uncertainty-aware motion planning, and learning-based dynamics and uncertainty modeling.
Our work lies at the intersection of these directions.

\subsection{Kinodynamic Motion Planning}
Kinodynamic motion planning addresses the problem of generating dynamically feasible trajectories that satisfy both geometric constraints and system dynamics.
These methods are suitable for robotic systems with complex dynamics including mobile robots~\cite{belief_tree}, drones~\cite{dirt}, and other systems with nontrivial constraints~\cite{aox, activepusher}.
Sampling-based methods, such as Rapidly-exploring Random Trees (RRT) and their variants, have become a dominant paradigm due to their probabilistic completeness and scalability in high-dimensional spaces~\cite{review-asym}.

To improve solution quality, asymptotically optimal kinodynamic planners have been developed. Methods such as SST~\cite{sst} and DIRT~\cite{dirt} achieve asymptotic optimality without requiring a steering function by maintaining a sparsified tree structure and pruning suboptimal nodes.
More generally, the AO-$x$ framework~\cite{aox} reformulates optimal kinodynamic planning as a sequence of feasible planning problems in an augmented state-cost space. Building on this, \cite{aorrt} formally proves asymptotic optimality of AO-RRT with running cost objectives.

Despite these advances, existing optimal planners primarily consider objectives defined as running costs along trajectories.
This limitation restricts the ability of planners to explicitly encode preferences over terminal states, such as goal quality, downstream task cost, or preferences among multiple feasible goal regions.
This gap motivates our extension of AO-RRT to incorporate terminal cost with formal guarantees.

\subsection{Uncertainty-Aware Motion Planning}
Uncertainty-aware planning aims to explicitly reason about uncertainty in system dynamics, sensing, and the environment.
Sampling-based approaches have been successfully adapted to this setting by planning in belief space, where a belief represents a probability distribution over the underlying state.
Belief-space planners maintain a search tree of belief nodes under uncertainty~\cite{rrbt}, and Gaussian belief tree~\cite{belief_tree} extends it to kinodynamic planning with chance constraints for partially observable systems.
A key challenge in belief-space planning is defining a valid metric, for which metrics such as the 2-Wasserstein distance have been proposed~\cite{belief-distance}.
POMDP-based methods further address long-horizon and online decision making under uncertainty~\cite{uncertainty_long_time, pomdp_uncertainty, pomdp_car}.

A parallel line of work focuses on incorporating risk and robustness into the planning process.
For example, chance-constrained planners~\cite{ccrrt, multi_cc} enforce safety constraints to avoid collisions and terminal constraints to reach the goal region with high probability.
Similarly, risk-aware planners incorporate uncertainty through different representations, including risk maps for perception uncertainty~\cite{risk-contours}, ambiguity sets for distributionally robust planning~\cite{drrrt}, and Monte Carlo estimation for trajectory optimization~\cite{mc_planning}.
Other work provides robustness for uncertain nonlinear systems~\cite{robust_rrt}, optimizes collision probability under environment uncertainty~\cite{collision_uncertainty}, and reasons about stochastic robustness for temporal-logic specifications~\cite{stochastic_robustness}.
%
%

However, existing sampling-based planners focus on feasibility, chance constraints, or trajectory-wise costs in belief space, such as cumulative length or risk.
The role of terminal cost in belief space, which captures the quality of the final belief state, remains underexplored with theoretical guarantees.

\subsection{Learning Dynamics and Uncertainty Models}
Accurate modeling of system dynamics and uncertainty is critical for effective planning under uncertainty, particularly in complex and contact-rich domains.
Learning-based approaches provide a flexible alternative to analytical models by fitting both dynamics and uncertainty directly from data~\cite{poking, activepusher}.
Gaussian process (GP) models, as used in learning-based task and motion planning~\cite{ltamp}, provide uncertainty estimates and enable sampling of actions with high success rate.
In deep learning settings, aleatoric uncertainty can be learned through probabilistic loss such as negative log-likelihood (NLL) or its variants, e.g., $\beta$-NLL~\cite{beta-nll}.
In low-data regimes, modeling epistemic uncertainty is also important for safe and informative planning. This can be achieved using approximate Bayesian methods such as Monte Carlo dropout~\cite{dropout-uncertainty}, with learned embeddings~\cite{active_regression}, or through learned uncertainty estimators~\cite{state_epistemic}.
More recently, evidential regression~\cite{evidential} has been proposed to jointly capture both uncertainties within a unified framework.

Learning-based models have been integrated with planning frameworks to enable uncertainty-aware decision making~\cite{activepusher, ltamp}.
However, such integrations use learned uncertainty only to bias action selections, rather than being tightly integrated with belief-space planners to explicitly propagate the belief dynamics.
In this work, we show how to train a neural network with a standard NLL objective to acquire the required belief dynamics model suitable for belief-space planning.

\section{Problem Formulation}
\label{sec:problem}

In this work, we consider kinodynamic planning problems under uncertainty.
Let $\X$ and $\U$ be the state space and control space respectively. The state space $\X$ is partitioned into two disjoint subsets
$\X = \Xobs \;\cup\; \Xfree$
, where $\Xobs$ is the obstacle region and $\Xfree$ is the collision-free region.
The system dynamics is given by the ordinary differential equation:
\begin{equation}
\label{eq:dynamics}
\dot x(t) = f\big(x(t), u(t)\big),
\end{equation}
where $t \in [0, \T]$ denotes time over the horizon \T, $x(t) \in \X$ is the system state and $u(t) \in \U$ is the corresponding control~\cite{review-sampling}.
%

\subsection{Kinodynamic Planning:}
The \emph{kinodynamic planning problem} is to determine a control function $\utraj:[0, \T] \to \U$ over a time horizon $\T$ such that the generated  state trajectory $\xtraj:[0, \T] \to \X$ satisfies:
\begin{equation}
\label{eq:kino_traj}
\begin{gathered} 
    x(0) = x_{\start} ,\; x(\T) \in \G, \\
    x(t) \in \Xfree ,\;
    u(t) \in \U,\;
    \dot{x}(t) = f\big(x(t), u(t)\big),\;
    \forall t \in [0,\T].
\end{gathered}
\end{equation}
where $x_\start$ denotes the starting state, and $\G$ denotes the goal region of the problem. For better readability, we use $x(\cdot) \equiv \xtraj(\cdot)$ and $u(\cdot) \equiv \utraj(\cdot)$.

In optimal kinodynamic motion planning, the planner also tries to optimize an objective function $\cost(\xtraj, \utraj)$ while searching for a feasible trajectory.
This objective typically measures the quality of a trajectory, which is expressed as an integral cost $\cost_\length$ accumulated along the trajectory:
\begin{equation}
\label{eq:cost_integral}
\cost_\length(\xtraj, \utraj)
=
\int_0^{\T} \length\big(x(t), u(t)\big) \, dt,
\end{equation}
where $\length: \mathcal{X} \times \mathcal{U} \rightarrow \Rplus$ is the running cost (e.g. path length or energy consumption).

In this work, we further consider a \emph{terminal cost} term $\terminal$ to optimize terminal-state quality, and consider the total cost $\cost_T$:
\begin{equation}
\label{eq:cost_terminal}
\cost_T(\xtraj, \utraj)
=
\int_0^{\T} \length \big(x(t), u(t)\big) \, dt
+ \terminal \big(x(\T)\big),
\end{equation}
%
where $\terminal: \mathcal{X} \rightarrow \Rplus$ is a terminal cost function that is evaluated based on the endpoint $x(\T)$. 
%

\subsection{Kinodynamic Planning in Belief Space} 
Under uncertainty, the system evolution is no longer fully deterministic.
Instead of planning directly in the state space $\mathcal{X}$, we plan in the \emph{belief space} $\mathcal{B}$, where each belief $b \in \mathcal{B}$ represents a probability distribution over $\mathcal{X}$ that captures the system’s uncertainty about its true state.
We assume that the belief over the true (but unknown) state can be approximated by a Gaussian distribution $b = \mathcal{N}(x, \cov)$, parameterized by its mean $x \in \X$ and covariance $\cov \in \SPD$, where $\SPD$ denotes the space of symmetric positive definite matrices.

The belief evolves according to Gaussian propagation under process and observation noise~\cite{lqg-mp}. In this work, we consider open-loop planning and focus on process uncertainty, omitting the observation update, such that the belief is propagated purely through the belief dynamics model:
\begin{equation}
\label{eq:belief_dyn}
\dot{b}(t) = f_b(b(t), u(t)),
\end{equation}
where $f_b$ represents the combined effect of motion dynamics with process uncertainty.
This form is analogous to the deterministic system dynamics in~\autoref{eq:dynamics}, but now the state variable $x(t)$ is replaced by the belief $b(t)$, which encodes both the expected state and its associated uncertainty.
The specific belief propagation models used in this work are introduced later in \autoref{sec:aorrt-belief-prop}.
Additionally, we demonstrate in \autoref{sec:modeling} that such a model can be directly learned from interaction data.

The \emph{belief-space kinodynamic planning problem} seeks a belief trajectory $\btraj:[0,\T]\!\to\!\mathcal{B}$ generated by a control function $\utraj$ such that, with high probability, the true trajectory $\xtraj$ induced by $\btraj$ satisfies the kinodynamic feasibility conditions in~\autoref{eq:kino_traj}.
In particular, we require the control and dynamics constraints to hold as in the deterministic setting, while collision avoidance and goal satisfaction are expressed probabilistically~\cite{ccrrt, belief_tree}:
\begin{equation}
\label{eq:belief_problem}
\begin{gathered}
    x(0) = x_{\start}, \\
    u(t) \in \U,\; 
    \dot{x}(t) = f\big(x(t), u(t)\big),\;
    \forall t \in [0,\T], \\
    \Pr\!\big( x(t) \in \Xfree \big) \ge p_\free ,\; \forall t \in [0,\T],\\
    \Pr\!\big( x(\T) \in \G \big) \ge p_\goal.
\end{gathered}
\end{equation}
where $p_\free$ and $p_\goal$ are the thresholds for collision-free execution and goal-reaching success, respectively.
Intuitively, the planner searches for a belief trajectory whose induced distribution over true trajectories remains largely within the feasible region of the deterministic problem.

Similar to~\autoref{eq:cost_terminal}, the belief-space objective evaluates trajectory quality at the distribution level, through both a running cost and a terminal cost:
\begin{equation}
\label{eq:belief_cost_terminal}
\cost_\B(\btraj, \utraj)
=
\int_0^{\T} \lengthb \big(b(t), u(t)\big) \, dt
+ \terminalb \big(b(\T)\big),
\end{equation}
where $\lengthb: \B \times \U \rightarrow \Rplus$ is the running cost in belief space and $\terminalb: \B \rightarrow \Rplus$ is the terminal belief cost.
For readability, we use $b(\cdot) \equiv \btraj(\cdot)$.
A key contribution of this work is to show that AO-RRT can optimize objectives with an explicit terminal cost (\autoref{eq:cost_terminal} and~\autoref{eq:belief_cost_terminal}), while preserving asymptotic optimality.
\section{\method: AO-RRT with A Terminal Cost}
\label{sec:aorrt-terminal}

In this section, we first introduce the core planning algorithm for \method, AO-RRT, originating from the AO-x family~\cite{aox}.
Although the original AO-x framework defines an optimization objective that includes a terminal cost, its asymptotic optimality analysis relies on the assumption that the underlying RRT is well-behaved, yet no rigorous proof is given when a terminal cost is present.
Moreover, this term is not incorporated or evaluated in any of its experiments.
A later extension~\cite{aorrt} establishes a formal asymptotic optimality proof for AO-RRT, but only for objectives consisting of an integrable running cost.

We extend previous theoretical analyses~\cite{aox, aorrt} by demonstrating that a terminal cost can be incorporated into AO-RRT while preserving asymptotic optimality.
Our analysis provides a complete and rigorous proof for optimizing both running and terminal costs under mild regularity assumptions.
%

\subsection{Algorithm Description}
\label{sec:aorrt-algo}

The key idea behind AO-RRT~\cite{aox} is to reformulate the optimal kinodynamic planning problem as a sequence of feasible planning problems in the \emph{augmented state space}.
Each iteration of feasible planning is performed under a progressively tighter cost bound, thereby converging toward the optimal solution over time.
To enable this formulation, the original state $x$ is augmented with the accumulated running cost $c$.
Accordingly, the state space is extended to $\Y = \X \times \Rplus$, with augmented state $y = (x, c) \in \Y$, while the control space remains $\U$.
Let $\aug$ denote the augmented dynamics:
\begin{equation}
\label{eq:aug_dynamics}
\begin{aligned}
\dot{y}(t)
=
\aug\big(y(t),u(t)\big)
=
\begin{bmatrix}
\dot{x}(t) \\
\dot{c}(t)
\end{bmatrix}
=
\begin{bmatrix}
f\big(x(t),u(t)\big) \\
\length\big(x(t),u(t)\big)
\end{bmatrix}.
\end{aligned}
\end{equation}

The pseudocode of \method with a single search tree is provided in \autoref{alg:aorrt}, where the key modification to AO-RRT regarding the terminal cost is highlighted in red.
It first initializes a search tree $G$ with start state $(x_{\start}, 0)$ in the augmented space, set the current best cost $c_{\text{best}}$ to $+\infty$, and set the terminal state of the previously found solution $y_\text{sol}$ to null (\autoref{alg:aorrt:init_start}–\autoref{alg:aorrt:init_end}).

At the beginning, no solution is found and the algorithm behaves identically to vanilla RRT. It only samples a random state and control with a step duration given the maximum control duration $\T_p > 0$ (\autoref{alg:aorrt:sampling_start}–\autoref{alg:aorrt:sampling_end}), and finds the nearest neighbor with the geometric distance metric using $\mathtt{NearestState}$ (\autoref{alg:aorrt:no_sol_start}–\autoref{alg:aorrt:no_sol_end}).
When a new solution $y_\text{sol}$ is found and a valid current best cost bound $c_{\text{best}}$ is updated, AO-RRT augments sampling with a cost coordinate, and $\mathtt{NearestStateCost}$ selects the nearest neighbor in the augmented space, focusing exploration on the sublevel set induced by the current cost bound $\{(x, c) : c < c_{\text{best}}\}$ (\autoref{alg:aorrt:itr_sol_start}–\autoref{alg:aorrt:itr_sol_end}). In this work, we define the augmented space distance as the L2 norm:
\begin{equation}
\label{eq:dist_aug}
D_y\left( y, y' \right) = \sqrt{D_x (x, x')^2 + |c - c'|^2}
\end{equation}
where $y = (x, c), y' = (x', c') \in \Y$ and $D_x (x, x')$ measures the geometric distance between $x$ and $x'$.

\begin{algorithm}
\caption{$\mathtt{\method}
\;(x_{\start}, \G, \X, \T_p, \U, \cost_\length, \terminal, \heuristics, N)
$}
\label{alg:aorrt}
\SetNoFillComment

$c_\text{best} \gets +\infty, \; y_\text{sol} \gets \varnothing$\;
\nllabel{alg:aorrt:init_start}
$G \gets \{\mathbb{V} \gets \{(x_\start, 0)\}, \mathbb{E} \gets \emptyset\} $\;
\nllabel{alg:aorrt:init_end}

\For{$i = 1, 2, ..., N$}{

    \tcc{Sampling}
    $x_{\text{rand}} \gets \mathtt{SampleState}(\X)$\;
    \nllabel{alg:aorrt:sampling_start}
    $\tau_{\text{rand}} \gets \mathtt{SampleDuration}(0, \T_p)$\;
    $u_{\text{rand}} \gets \mathtt{SampleControl}(\U)$\;
    \nllabel{alg:aorrt:sampling_end}

    \tcc{No solution found. Use state distance}
    \If{$y_\text{sol} = \varnothing$}{
    \nllabel{alg:aorrt:no_sol_start}
        $y_{\text{near}} \gets \mathtt{NearestState}(x_{\text{rand}}, G)$\;
        \nllabel{alg:aorrt:no_sol_end}
    }

    \tcc{Best cost set. Use augmented distance}
    \Else{
    \nllabel{alg:aorrt:itr_sol_start}
        $c_{\text{rand}} \gets \mathtt{SampleCost}(0, c_\text{best})$\;
        $y_{\text{near}} \gets \mathtt{NearestStateCost}((x_{\text{rand}}, c_{\text{rand}}), G)$\;
        \nllabel{alg:aorrt:itr_sol_end}
    }

    \tcc{Tree extension}
    $(\xtraj_{\text{seg}}, \utraj_{\text{seg}}) \gets \mathtt{Propagate}(y_{\text{near}}.x, u_{\text{rand}}, \tau_{\text{rand}})$\;
    \nllabel{alg:aorrt:prop_start}
    $x_{\text{new}} \gets x_{\text{seg}}(\tau_{\text{rand}}) $\;
    $c_{\text{new}} \gets y_{\text{near}}.c + \cost_\length(\xtraj_{\text{seg}}, \utraj_{\text{seg}})$\;
    \nllabel{alg:aorrt:prop_end}

    \tcc{Validity check}
    \If{$\mathtt{IsInvalid}(\xtraj_{\text{seg}}, \X)$}{
    \nllabel{alg:aorrt:val_start}
        \textbf{continue} \tcp{Reject invalid motion}
    }
    \nllabel{alg:aorrt:val_end}

    \If{$c_{\text{new}} + \heuristics(x_\text{new}) \ge c_\text{best}$}{
    \nllabel{alg:aorrt:c_bound_start}
        \textbf{continue} \tcp{Reject cost above bound}
        \nllabel{alg:aorrt:c_bound_end}
    }

    \tcc{Update tree structure}
    $y_{\text{new}} \gets (x_{\text{new}},\, c_{\text{new}})$\;
    \nllabel{alg:aorrt:tree_start}
    $\mathbb{V} \gets \mathbb{V} \;\cup\; \{y_{\text{new}}\} $\;
    $\mathbb{E} \gets \mathbb{E} \;\cup\; \{\text{Edge}(y_{\text{near}}, y_{\text{new}})\} $\;
    \nllabel{alg:aorrt:tree_end}

    \tcc{Goal and cost bound check}
    \If{$x_{\text{new}} \in \G 
    \;\land\;$
    \textcolor{rred}{
    $c_{\text{new}} + \terminal(x_{\text{new}}) < c_\text{best}$
    }
    }{
    \nllabel{alg:aorrt:goal_start}

        \tcc{Add terminal cost in the end}
        \textcolor{rred}{
        $c_\text{best} \gets c_{\text{new}} + \terminal(x_{\text{new}})$
        }\;
        \nllabel{alg:aorrt:goal_end}
        
        $y_{\text{sol}} \gets y_{\text{new}}$\;
        \nllabel{alg:aorrt:sol}
        
        $\mathtt{PruneTreeAboveCost}(G,\, c_\text{best})$\;
        \nllabel{alg:aorrt:prune}
    }
}
\Return $\mathtt{PathToRoot}(y_\text{sol})$\;
\nllabel{alg:aorrt:path}

\end{algorithm}

From the selected vertex, the algorithm propagates the dynamics to obtain $(x_{\text{new}}, c_{\text{new}})$ (\autoref{alg:aorrt:prop_start}–\autoref{alg:aorrt:prop_end}). 
Although this work introduces an additional terminal cost $\terminal$ in the overall objective, we emphasize that $\terminal$ is not embedded within the cost variable $c$ of the augmented dynamics.
The extended state includes only the running cost $\length$, and $\terminal$ is considered when a trajectory reaches the goal (\autoref{alg:aorrt:goal_start}–\autoref{alg:aorrt:goal_end}).
Motion segments that are invalid are discarded (\autoref{alg:aorrt:val_start}–\autoref{alg:aorrt:val_end}).
In addition, a segment $x_\text{new}$ is rejected early if its accumulated cost plus an admissible heuristic $\mathcal{H}(x_\text{new})$ cannot improve the current best solution, i.e.,
$c_{\text{new}} + \heuristics(x_{\text{new}}) \ge c_{\text{best}}$
(\autoref{alg:aorrt:c_bound_start}–\autoref{alg:aorrt:c_bound_end}). Valid vertices are then inserted into the tree (\autoref{alg:aorrt:tree_start}–\autoref{alg:aorrt:tree_end}).

When a newly added state reaches the goal region $\G$, the resulting total cost $c_{\text{new}} + \terminal(x_{\text{new}})$ will be evaluated.
If it improves upon the best known cost, the solution will be updated and the cost bound will be set to this new total cost (\autoref{alg:aorrt:goal_start}–\autoref{alg:aorrt:sol}).
Since all costs are strictly positive, it is safe and efficient to invoke $\mathtt{PruneTreeAboveCost}$ to remove all vertices and edges whose running cost already exceeds the updated bound, thereby avoiding expansions that cannot lead to a better solution (\autoref{alg:aorrt:prune}).
After $N$ iterations, the algorithm returns the best path from the start state to $y_\text{sol}$ (\autoref{alg:aorrt:path}).

In the next sections we rigorously prove that \method (AO-RRT with terminal cost) is asymptotically optimal.

\subsection{Proof Preliminary}
\label{sec:aorrt-terminal-pre}
For simplicity and clarity, in this proof, we adopt the Euclidean norm $\| \cdot \|$ when measuring differences in state, control, and dynamics, and the absolute value $| \cdot |$ when measuring differences in cost. The results naturally extend to the weighted norm or other valid metric~\cite{aorrt} that satisfies the Lipschitz continuity assumptions below.
For any measurable set $S \subseteq \mathbb{R}^d$, let $\vol{S}$ denote its Lebesgue measure, i.e., its $d$-dimensional volume. 
For any state $x \in \X$ and radius $r > 0$, we denote the $r$-radius closed ball centered at $x$ as:
\begin{equation}
B_r(x) = \{\, x' \in \X \mid  \|x - x' \| \le r \,\}.
\end{equation}
%

We begin by stating the assumptions required for the optimality proof. Specifically, we assume Lipschitz continuity for the system dynamics and objective functions.

\begin{assumption}[Lipschitz Continuity of the System Dynamics]
\label{as:lip-system}
The system dynamics $f$ is assumed to be Lipschitz continuous in both the state and control variables. Formally, $\exists K^f_x, K^f_u \ge 0$ such that, $\forall x, x' \in \X$, $\forall u, u' \in \U$:
\begin{equation}
\begin{gathered}
\label{eq:lip_f}
\|f(x,u) - f(x',u)\| \le K^f_x\|x - x'\|, \\
\|f(x,u) - f(x,u')\| \le K^f_u\|u - u'\|.
\end{gathered}
\end{equation}
This assumption ensures that small perturbations in the state or control lead to proportionally bounded variations in the system dynamics.
\end{assumption}

\begin{assumption}[Lipschitz Continuity of the Running Cost]
\label{as:lip-running}
The running cost $\length(x,u)$ is assumed to be Lipschitz continuous, that is, 
$\exists K^\length_x, K^\length_u \ge 0$ such that $\forall x, x' \in \X$, $\forall u, u' \in \U$:
\begin{equation}
\begin{gathered}
\label{eq:lip_l}
|\length(x,u) - \length(x',u)| \le K^\length_x\|x - x'\|, \\
|\length(x,u) - \length(x,u')| \le K^\length_u\|u - u'\|.
\end{gathered}
\end{equation}
This assumption guarantees that the cost rate varies smoothly along nearby trajectories.
\end{assumption}
If the system dynamics $f$ and running cost $\ell$ both satisfy Lipschitz continuity, it is shown in~\cite{aorrt} that its corresponding augmented system $\aug$ also satisfies Lipschitz continuity with some constants $K_y, K_u \ge 0$:
\begin{equation}
\begin{gathered}
\label{eq:lip_z}
\|\aug(y,u) - \aug(y',u)\| \le K_y \|y - y'\|, \\
\|\aug(y,u) - \aug(y,u')\| \le K_u \|u - u'\|.
\end{gathered}
\end{equation}

\begin{assumption}[Lipschitz Continuity of the Terminal Cost]
\label{as:lip-terminal}
Finally, the terminal cost introduced in this work is assumed to be Lipschitz continuous in a neighborhood of the goal region $\G$. Specifically, 
$\exists K_\terminal \ge 0$ such that $\forall x, x' \in \G$:
\begin{equation}
\label{eq:lip_terminal}
|\terminal(x) - \terminal(x')| \le K_\terminal\|x - x'\|.
\end{equation}
This implies that small deviations in the terminal state result in proportionally bounded differences in terminal cost.
\end{assumption}

To support the theoretical analysis, we introduce several definitions and notations, following prior work on optimal kinodynamic planning~\cite{sst, aorrt}. Sampling-based kinodynamic planners are typically implemented such that the control function is piecewise-constant:

\begin{definition}[Piecewise-constant Control]
A control function $\utraj$ is \emph{piecewise-constant} with resolution $\tau > 0$ if it is represented as a finite concatenation of constant control segments. Each segment $i$ is defined by a function: $\utraj_i
: [0, k_i \cdot \tau] \rightarrow u_i$, where $u_i \in \U$ is fixed and $k_i \in \mathbb{N}_+$.
\end{definition}

\begin{definition}[Valid Trajectory]
A trajectory $\xtraj$ generated by a piecewise-constant control function $\utraj$ is said to be \emph{valid} if it satisfies the kinodynamic constraints in \autoref{eq:kino_traj}.
\end{definition}

\begin{definition}[Obstacle and Goal Clearance]
The \emph{obstacle clearance} $\delta_o$ of a valid trajectory $\xtraj$ is the minimum distance from obstacles over all states in $\xtraj$, i.e., maximum value $\delta_o$ such that:
\begin{equation}
B_{\delta_o}(x) \subset \Xfree, \forall x \in \xtraj .
\end{equation}
Similarly, the \emph{goal clearance} $\delta_g$ of a valid trajectory $\xtraj$ is the minimum distance from its terminal state $x(\T)$ to the boundary of goal region $\G$, that is, the maximum value $\delta_g$ such that:
\begin{equation}
\ball_{\delta_g}(x(\T))\subset \G .
\end{equation}
\end{definition}

\begin{definition}[Robust and $\delta$-robust Trajectory]
A valid trajectory is \emph{robust} if there exists a clearance $\delta > 0$ such that both its obstacle clearance $\delta_o$ and goal clearance $\delta_g$ are at least $\delta$. 
When the value of $\delta$ is specified, the trajectory $\xtraj_\delta$ is called \emph{$\delta$-robust}.
\end{definition}

\begin{definition}[Asymptotic Optimality]
Given a problem, let $c^\ast$ denote the minimum achievable cost among all robust trajectories. Let $C^{\mathtt{ALG}}_n$ denote a random variable that represents the best cost value among all solutions from algorithm $\mathtt{ALG}$ at iteration $n$. $\mathtt{ALG}$ is \emph{asymptotically optimal} if
\begin{equation}
\label{eq:ao}
\Pr \left( \lim_{n\to\infty} C^{\mathtt{ALG}}_n = c^\ast \right) = 1,
\end{equation}
that is, $C^{\mathtt{ALG}}_n$ converges to $c^\ast$ \emph{almost surely}.
\end{definition}

\subsection{Proof of Asymptotical Optimality}
\label{sec:aorrt-terminal-proof}
\begin{figure}[!ht]
    \centering
    \includegraphics[width=\linewidth]{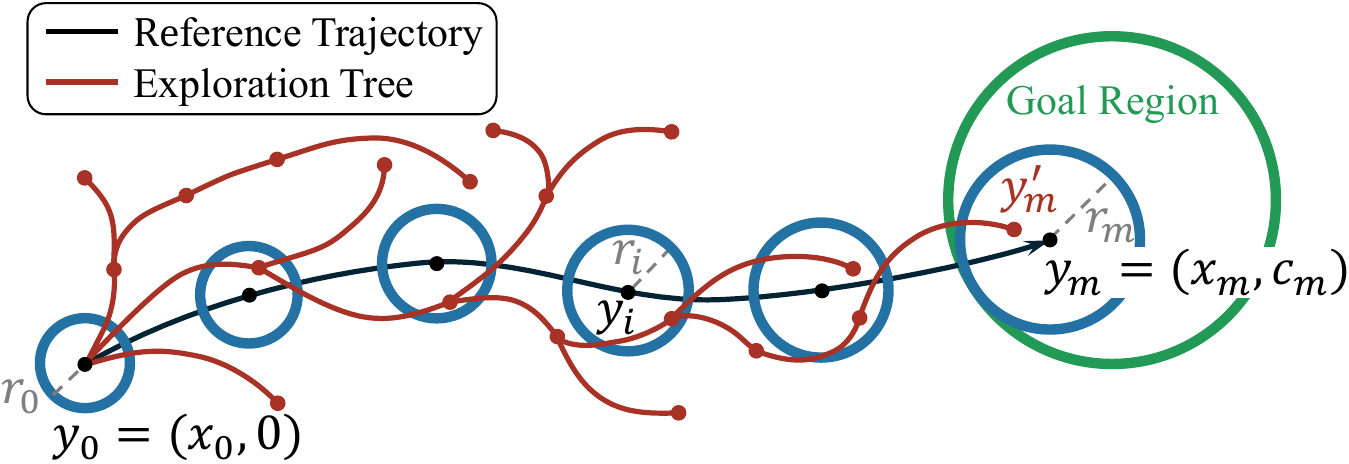}
    \caption{Illustration of asymptotic optimality proof (\autoref{thm:main}) of AO-RRT. A sequence of balls with increasing radii is placed along a $\delta$-robust reference trajectory from $y_0$ to $y_m$. AO-RRT can repeatedly select a vertex inside each ball and propagate it into the next. Reaching the terminal ball yields a trajectory whose total cost is within a factor of the reference cost.}
    \label{fig:ao_rrt_theorem1}
\vspace{-1ex}
\end{figure}

\begin{figure*}[!ht]
    \centering
    \begin{subfigure}[t]{0.28\textwidth}
        \centering
        \includegraphics[height=0.15\textheight]{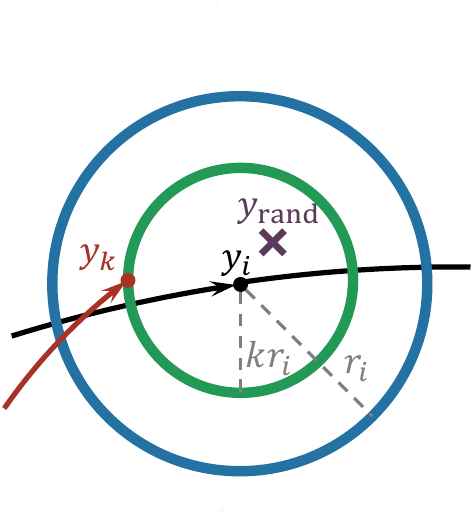}
        \caption{Illustration of \autoref{lem:select}. If the tree contains a vertex $y_k$ inside $B_{k r_i}(y_i)$, where $k\in(0,1)$, then with positive probability, AO-RRT will select a vertex within $B_{r_i}(y_i)$ for propagation.}
        \label{fig:ao_rrt_lemma1}
    \end{subfigure}
    \hfill
    \begin{subfigure}[t]{0.4\textwidth}
        \centering
        \includegraphics[height=0.15\textheight]{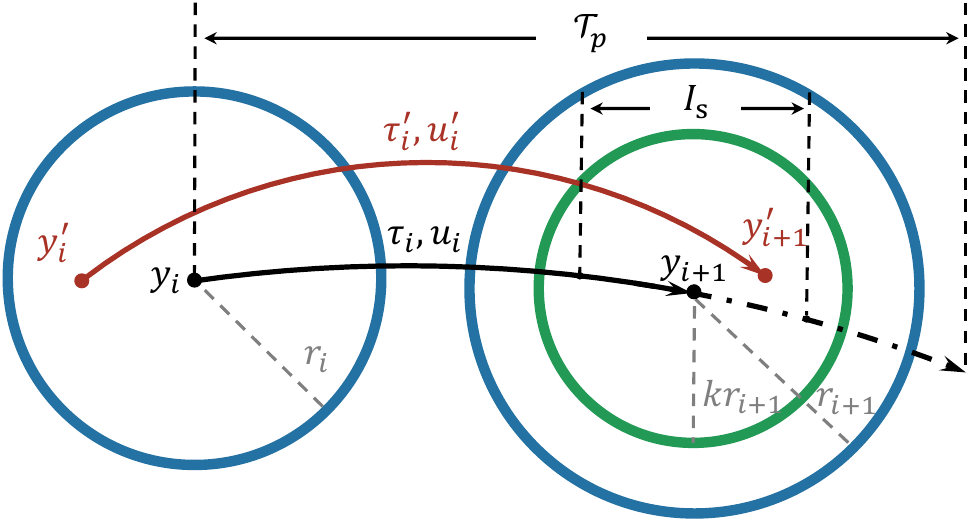}
        \caption{Illustration of \autoref{lem:propagate}. Starting from a perturbed starting state $y_i' \in B_{r_i}(y_i)$, with positive probability, AO-RRT can sample a control–duration pair $(u_i',\tau_i')$, near the reference pair $(u_i,\tau_i)$, such that the propagated state $y_{i+1}'$ remains within $B_{k r_{i+1}}(y_{i+1})$, where $k\in(0,1)$.}
        \label{fig:ao_rrt_lemma2}
    \end{subfigure}
    \hfill
    \begin{subfigure}[t]{0.28\textwidth}
        \centering
        \includegraphics[height=0.15\textheight]{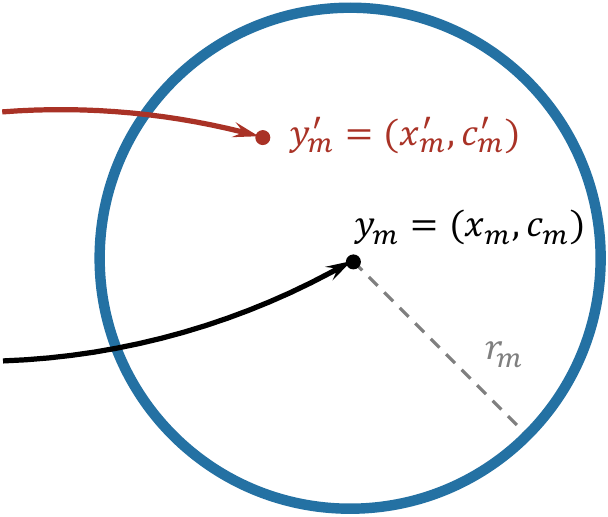}
        \caption{Illustration of \autoref{lem:terminal}~\textbf{(New}). Any terminal augmented state $y_m'$ inside $B_{r_m}(y_m)$ has total cost $c_T' = c_m' + \terminal(x_m')$ bounded by $(1+\varepsilon)c_T$, where $c_T = c_m + \terminal(x_m)$ is the reference total cost.}
    \end{subfigure}
    \caption{Illustrations of AO-RRT Lemmas}
    \label{fig:aorrt_lemmas}
\vspace{-2ex}
\end{figure*}
%
We now establish the asymptotic optimality of \method i.e., AO-RRT with a terminal cost.
The proof follows the general structure used in prior analysis of asymptotically optimal kinodynamic planners~\cite{sst, pcrrt, aorrt, robust_rrt}, while incorporating a new terminal-cost argument specific to our setting.
As illustrated in \autoref{fig:ao_rrt_theorem1}, the key idea is to place a finite sequence of balls with increasing radii along any $\delta$-robust reference trajectory in the augmented space $\Y$.
We then show that AO-RRT can, with positive probability, repeatedly select a vertex from the current ball (\autoref{lem:select}) and propagate it into the next ball (\autoref{lem:propagate}).
Repeating this argument for all balls yields a positive probability of reaching the final ball, thereby producing a trajectory whose total cost is arbitrarily close to the reference one (\autoref{lem:terminal}).
We then state the main theorem.

\begin{theorem}[Asymptotic Optimality]
\label{thm:main}
Assume that the system dynamics, running cost, and terminal cost are all Lipschitz continuous and let $\xtraj_\delta$ be any robust trajectory generated with piecewise-constant control $\utraj_\delta$.
Then $\forall \varepsilon \in (0,1)$, $\exists a, b > 0$, $n_0 \in \mathbb{N}$, s.t. $\forall n>n_0$:
\begin{equation}
\Pr \big(\cost_T(\xtraj_n, \utraj_n) > (1 + \varepsilon)\, \cost_T(\xtraj_\delta, \utraj_\delta) \big)
\le a \cdot e^{-bn},
\end{equation}
where ($\xtraj_n$, $\utraj_n$) is the best solution returned by AO-RRT after $n$ iterations.

Consequently, let $c^\ast$ denote the minimum achievable cost among all such trajectories $\xtraj_\delta$, then AO-RRT is (almost surely) asymptotically optimal with respect to cost $\cost_T$, i.e., 
\begin{equation}
\label{eq:ao}
\Pr \left( \lim_{n\to\infty} \cost_T(\xtraj_n, \utraj_n) = c^\ast \right) = 1.
\end{equation}
\end{theorem}

The proof of \autoref{thm:main} relies on the following three lemmas, illustrated in \autoref{fig:aorrt_lemmas}.
We adapt \autoref{lem:select} and \autoref{lem:propagate} from~\cite{pcrrt, sst} and provide proofs that are refined to improve precision and clarity.
We then introduce a new lemma \autoref{lem:terminal} that captures the effect of the terminal cost.

\begin{lemma}
\label{lem:select}
Suppose the tree contains a vertex $y_k \in B_{kr_i}(y_i)$ for some $k \in (0,1)$. 
Given a random sample $y_{\text{rand}} \in \Y$, the probability that AO-RRT selects a vertex $y_{\text{near}}$ as the nearest neighbor for propagation inside $B_{r_i}(y_i)$ is positive and satisfies
\begin{equation}
    p_i^{\text{sel}} \coloneq
    \Pr \big(y_{\text{near}} \in B_{r_i}(y_i) \big) \ge 
    \frac{\vol{ B_{(1 - k)r_i/2}(y_i) }}{\vol{\Y}} 
    > 0.
\end{equation}
\end{lemma}

If the tree already contains a vertex $y_k$ close to $y_i$, then any random sample $y_\text{rand}$ falling in a sufficiently small neighborhood around $y_i$ will not select a vertex outside $B_{r_i}(y_i)$ as the nearest neighbor.
Thus, selection from the ball $B_{r_i}(y_i)$ occurs with nonzero probability.
A full proof is given in~\appref{app:aorrt-lem-select-proof}.

\begin{lemma}
\label{lem:propagate}
Consider a trajectory segment from $y_i$ to $y_{i+1}$, generated by applying a constant control $u_i \in \U$ for a duration $\tau_i \le \T_p$, where $\T_p$ is the maximum propagation time.
Let $K_y$ be the Lipschitz constant of the augmented dynamics with respect to the state.
Let $B_{r_i}(y_i)$ and $B_{r_{i+1}}(y_{i+1})$ be two balls centered at $y_i$ and $y_{i+1}$, respectively, with increasing radii satisfying $r_{i+1} = r_i (e^{K_y\T_p} + \beta) /\, k$ for some $\beta > 0$ and $k \in (0, 1)$.

Then for any $y_i' \in B_{r_i}(y_i)$, there exist a control neighborhood $B_{\Delta u_i}(u_i) \subset \U$ with $\Delta u_i > 0$ and a time interval $I_s \subset (0,\T_p]$ of nonzero measure such that, for all
$u_i' \in B_{\Delta u_i}(u_i)$ and
$\tau_i' \in I_s, $
the propagated state $y_{i+1}'$ obtained from $y_i'$ under $(u_i', \tau_i')$
lies in $B_{k r_{i+1}}(y_{i+1})$.
Consequently, AO-RRT can sample a control-duration pair and reach $B_{k r_{i+1}}(y_{i+1})$ with positive probability:
\begin{equation}
\begin{aligned}
p_i^{\text{ext}} 
&\coloneq
\Pr\!\left( y_{i+1}' \in B_{k r_{i+1}}(y_{i+1}) \right)\\
&\ge
\frac{\vol{I_s}}{\vol{(0,\T_p]}}
\cdot
\frac{\vol{B_{\Delta u_i}(u_i)}}{\vol{\U}}\\
&> 0.
\end{aligned}
\end{equation}
\end{lemma}
Because the dynamics and running cost are Lipschitz continuous, small perturbations in the initial state, control, and its duration induce bounded deviations in the resulting trajectory.
Starting from a vertex $y_i'$ inside the ball $B_{r_i}(y_i)$, AO-RRT can reach the next ball with nonzero probability by sampling a sufficiently close control-duration pair $(u_i', \tau_i')$ to the original pair $(u_i, \tau_i)$.
The full proof is provided in~\appref{app:aorrt-lem-propagate-proof}.

\begin{lemma}
\label{lem:terminal}
Let $y_m = (x_m, c_m)$ be the terminal state of a $\delta$-robust trajectory with total cost $c_T = c_m + \terminal_m$, where $\terminal_m = \terminal(x_m)$. Let $K_\terminal$ be the Lipschitz constant of the terminal cost function. For any $\varepsilon \in (0,1)$, define
\begin{equation}
\label{eq:lem3_r}
r_m = \min(\delta, \dfrac{\varepsilon \, c_T}{1 + K_\terminal}).
\end{equation}
For any other robust trajectory whose terminal state is
$y_m' = (x_m', c_m') \in B_{r_m}(y_m)$ with terminal cost
$\terminal_m' = \terminal(x_m')$ and total cost $c_T' = c_m' + \terminal_m'$, we always have:
\begin{equation}
c_T' \le  (1 + \varepsilon) \, c_T.
\end{equation}
\end{lemma}

This lemma captures the key intuition behind AO-RRT.
Once the planner reaches a sufficiently small neighborhood of the reference terminal in the augmented space, it not only reaches the goal region, but also attains a cost close to that of the reference.
In our new setting with terminal cost, we show that by choosing the terminal radius $r_m$ in a particular way, any robust trajectory that ends inside the terminal ball $B_{r_m}(y_m)$ has a total cost within a factor $(1 + \varepsilon)$ of the reference cost $c_T$.
Similarly, the full proof of this lemma is provided in~\appref{app:aorrt-lem-terminal-proof}.

Now we prove ~\autoref{thm:main} with these lemmas.
\begin{proof}
As shown in \autoref{fig:ao_rrt_theorem1}, fix any $\delta$-robust trajectory $\xtraj_\delta$ with piece-wise constant controls $\utraj_\delta$ and total cost $\cost_T(\xtraj_\delta,\utraj_\delta)$.
Let $\ytraj_\delta$ denote its corresponding augmented trajectory with start $y_0 = (x_0,0)$ and terminal $y_m = (x_m,c_m)$.

Pick any $\varepsilon \in (0,1)$, and choose $r_m$ according to \autoref{lem:terminal}, i.e., $r_m = \min(\delta, \tfrac{\varepsilon \, c_T}{1 + K_\terminal})$.
Starting from this terminal ball $B_{r_m}(y_m)$, we place a finite sequence of balls $B_{r_0}(y_0), \dots ,B_{r_i}(y_i), \dots ,B_{r_m}(y_m)$ backward along the trajectory $\ytraj_\delta$.
Each consecutive pair corresponds to one constant-control segment, and their radii are chosen to satisfy the condition required by \autoref{lem:propagate}, i.e. $r_{i+1} = r_i (e^{K_y\T_p} + \beta) /\, k$.
Since each radius is obtained from the next one by a positive multiplicative factor, all balls have nonzero volume.

Now consider any stage $i \in \{0, \dots, m - 1\}$. By \autoref{lem:select}, if the tree contains a vertex in $B_{k r_i}(y_i)$, then in one iteration, AO-RRT selects a vertex in $B_{r_i}(y_i)$ with probability $p_i^{\text{sel}} > 0$.
Conditioned on such a selection, \autoref{lem:propagate} implies that, with probability $p_i^{\text{ext}} > 0$, the extension can reach $B_{k r_{i+1}}(y_{i+1})$.
Therefore, the probability of successfully advancing from stage $i$ to stange $i+1$ in one iteration is lower bounded by
\begin{equation}
p_i \coloneqq p_i^{\text{sel}} \cdot p_i^{\text{ext}} \;>\; 0.
\end{equation}
The induction is initialized automatically since $y_0 \in B_{k r_0}(y_0)$. Thus, the above argument applies in every stage, yielding a sequence of transitions from $B_{r_0}(y_0)$ to $B_{r_{m}}(y_{m})$, where each succeeds with probability at least $p \coloneqq \min_{i=0,\dots,m-1} p_i > 0$.

We consider each iteration as a Bernoulli trial that attempts to make progress from stage $i$ to the next stage $i+1$. To traverse the entire sequence, AO-RRT needs $m$ such successful transitions. 
Let $S_n$ denote the number of successful transitions over $n$ iterations.
Then the event “AO-RRT has not reached $B_{r_m}(y_m)$ after $n$ iterations” implies $S_n < m$.
Since each iteration has success probability at least $p_i$, the number of successful transitions can be stochastically lower bounded by a binomial random variable with parameters $(n,p)$. 
Standard tail bounds then imply that there exist constants $a, b>0$ and $n_0\in\mathbb{N}$ such that, for all $n > n_0$:
\begin{equation}
\Pr\left(
    S_n < m
\right)
\le
a e^{-b n}.
\end{equation}

If $S_n \ge m$, then the tree has successfully traversed the entire sequence of balls, and contains a terminal state $y_m' \in B_{r_m}(y_m)$.
By \autoref{lem:terminal}, the corresponding trajectory has total cost:
\begin{equation}
\cost_T(\xtraj',\utraj')
\le
(1+\varepsilon)\,\cost_T(\xtraj_\delta,\utraj_\delta).
\end{equation}
By contraposition, if the cost of the best solution $\cost_T(\xtraj_n, \utraj_n)$ still exceeds $(1+\varepsilon)\,\cost_T(\xtraj_\delta,\utraj_\delta)$, 
then AO-RRT cannot have completed all $m$ successful transitions, i.e., $S_n<m$.
Therefore, $\exists a, b > 0$, $n_0 \in \mathbb{N}$, s.t. $\forall n>n_0$:
\begin{equation}
\Pr \big(
\cost_T(\xtraj_n, \utraj_n)
>
(1 + \varepsilon)\, \cost_T(\xtraj_\delta, \utraj_\delta)
\big)
\le a e^{-b n},
\end{equation}
which proves the first part of \autoref{thm:main}.

Let $c^\ast$ denote the minimal achievable total cost over all robust trajectories, and select a robust trajectory $(\xtraj^\ast,\utraj^\ast)$ achieving this minimum, i.e.,
$\cost_T(\xtraj^\ast,\utraj^\ast)=c^\ast$.
For any $\varepsilon>0$, the argument above shows that there exist constants $a, b>0$ and $n_0\in\mathbb{N}$ such that, for all $n>n_0$:
\begin{equation}
\Pr\!\big(
\cost_T(\xtraj_n,\utraj_n)
>
(1+\varepsilon)\, c^\ast
\big)
\le a e^{-bn}.
\end{equation}
Thus, as $n \to \infty$:
\begin{equation}
\lim_{n\to\infty}
\Pr\!\left(
\cost_T(\xtraj_n,\utraj_n)
>
(1+\varepsilon)\, c^\ast
\right)
= 0,
\end{equation}
which establishes asymptotic optimality in probability~\cite{review-asym}.

Since $\sum_{n=1}^\infty a e^{-bn} < \infty$, the Borel-Cantelli lemma implies that the event $\{
\cost_T(\xtraj_n,\utraj_n) > (1+\varepsilon) \,c^\ast
\}$
can occur only finitely many times. Thus:
\begin{equation}
\Pr\!\left(
\cost_T(\xtraj_n,\utraj_n)
>
(1+\varepsilon)\, c^\ast
\text{ infinitely often}
\right)
=0.
\end{equation}
Hence, AO-RRT  with a total cost $\cost_T$ is (almost-surely) asymptotically optimal~\cite{review-asym}:
\begin{equation}
\Pr\!\left(
\lim_{n\to\infty}
\cost_T(\xtraj_n,\utraj_n)
= c^\ast
\right)=1.
\end{equation}
\end{proof}

\section{\method in Belief Space}
\label{sec:aorrt-belief}
In practical robotic systems, execution is subject to process uncertainty arising from modeling errors, perception noise, and stochastic disturbances. As a result, the system evolution is no longer deterministic, and planning only in the state space may produce trajectories that are difficult to execute reliably.

A principled solution is to perform planning in the belief space, where each belief represents a distribution over possible system states. Accordingly, rather than planning in the state-cost space $\Y = \X \times \Rplus$, \method can also plans in the belief state-cost space $\Yb = \B \times \Rplus$, where $\B$ denotes the space of probability measures over $\X$.
The asymptotic optimality of AO-RRT with terminal cost naturally carries over to $\Yb$ if the Lipschitz continuity conditions required by the analysis continue to hold in the belief space. 
We prove that, under mild assumptions on the stochastic dynamics, these conditions are satisfied when $\B$ is equipped with the Wasserstein metric.
%

\subsection{Belief-space Metric}
\label{sec:aorrt-belief-metric}

Following prior work on belief-space metrics~\cite{belief_tree, belief-distance}, we adopt the 2-Wasserstein distance (also known as earth mover's distance) as the proper distance metric in $\B$.
For simplicity, we will refer to it as Wasserstein distance.
The Wasserstein distance measures the minimum transport cost required to transform one probability distribution into another. Given a valid metric $D_\X$ in space $\X$, the Wasserstein distance is defined as:
\begin{equation}
\begin{aligned}
\label{eq:wasserstein}
W_2^2(b, b')
&=
\inf_{\gamma \in \Gamma(b, b')}
\mathbb{E} [D_\X(X, X')^2]
\\&=
\inf_{\gamma \in \Gamma(b, b')}
\int D_\X(x, x')^2 \, d\gamma(x, x'),
\end{aligned}
\end{equation}
where $\Gamma(b, b')$ denotes the set of all couplings (joint distributions) $\gamma$ whose marginal distributions are $b$ and $b'$ respectively.
Let $(X, X') \sim \gamma$ be random variables drawn from such a coupling so that $X \sim b, X' \sim b'$.
The infimum selects the optimal coupling $\gamma*$ that minimizes the expected transform cost, as shown in~\autoref{fig:wasserstein}.

\begin{figure}[!h]
    \centering
    \includegraphics[width=\linewidth]{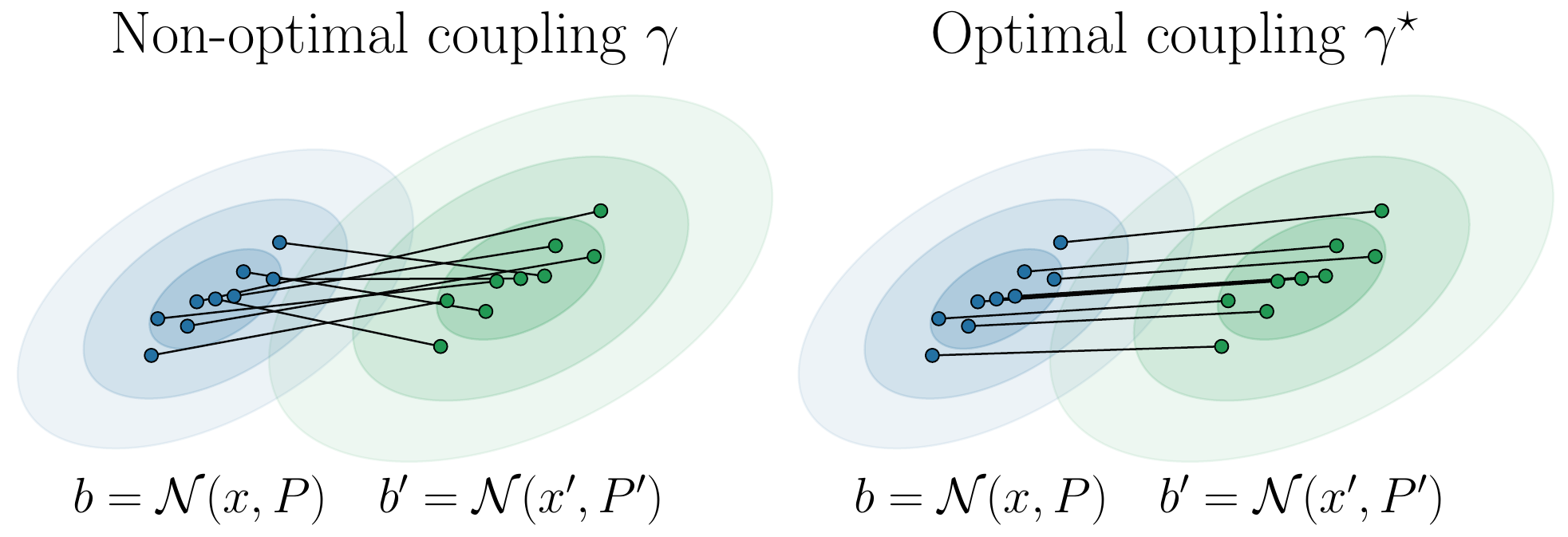}
    \caption{Illustration of Wasserstein distance between two Gaussian beliefs. Line segments denote pairings between representative locations of probabilistic mass in the state space. The optimal coupling minimizes the expected squared transport cost.}
    \label{fig:wasserstein}
\vspace{-2ex}
\end{figure}

For Gaussian beliefs, Wasserstein distance $W_2$ captures both changes in the mean and covariance of the state distribution, and admits a proper metric with closed-form expression.
Given two Gaussian beliefs $b = \mathcal{N}(x, P)$ and $b' = \mathcal{N}(x', P')$:
\begin{equation}
\label{eq:w2-gaussian}
W_2^2(b,b')
=
\|x - x'\|^2
+
\mathrm{Tr}\, \big(
P + P' - 2\,(P'^{\frac12} P P'^{\frac12})^{\frac12}
\big),
\end{equation}
where $\mathrm{Tr}(\cdot)$ is the trace of a matrix and $P^{\frac12}$ denotes the principal square root of $P$. The second term is also known as Bures metric between $P$ and $P'$.
%

\subsection{Belief-space Dynamics and Costs}
\label{sec:aorrt-belief-dyn}

Since controls are applied in piecewise-constant segments, we consider a general stochastic system of discrete-time form:
\begin{equation}
\label{eq:disc-stochastic-system}
\begin{gathered}
x_{i+1} = F(x_i, u_i) \diamond \eta_i,
\\
\eta_i \sim \mathcal{N}\big(0, Q(x_i, u_i)\big).
\end{gathered}
\end{equation}
where $F : \X \times \U \to \X$ is a deterministic transition map, $\eta_i$ is a random Gaussian noise term, and $\diamond$ specifies how noise is injected (e.g. $\diamond$ is addition $+$ in Euclidean space).
The process uncertainty may be heteroscedastic, i.e., the noise covariance $Q(x_i,u_i)$ may depend on the current state and applied control.
We approximate the belief state by a Gaussian distribution $b = \mathcal{N}(x, P)$.
Under the model in \autoref{eq:disc-stochastic-system}, the belief evolution is described by a discrete-time Markov process~\cite{belief_tree}:
\begin{equation}
\label{eq:belief_dyn}
b_{i+1} = F_b(b_i, u_i),
\end{equation}
where $F_b$ jointly propagates the mean and covariance of the belief state.
The cost function is then expressed as:
\begin{equation}
\label{eq:belief_cost}
\cost_\B(\btraj, \utraj)
=
\sum\nolimits_{i=0}^{m-1} \lengthb \big(b_i, F_b(b_i, u_i)\big)
+ \terminalb \big(b_m\big),
\end{equation}
where $\lengthb: \B \times \B \rightarrow \Rplus$ and $\terminalb: \B \rightarrow \Rplus$ are the belief-space running cost and terminal cost in discrete form, and $m$ is the number of steps in the trajectory.

We aim to show that the belief dynamics (\autoref{eq:belief_dyn}), belief-space running cost and terminal cost (\autoref{eq:belief_cost}) are Lipschitz continuous with respect to the Wasserstein metric.
These properties are sufficient to ensure that the theoretical guarantees of AO-RRT in~\autoref{sec:aorrt-belief} extend to the belief state-cost space.

Let $D_\X$ be a valid metric in state space $\X$.
With discrete-time representation, the Lipschitz conditions of the system dynamics in \autoref{as:lip-system} takes the form:
\begin{equation}
\label{eq:lip-system-F}
\begin{gathered}
D_\X\big(F(x,u),\, F(x',u)\big)
\le
K_x^F\, D_\X(x, x'),
\\
D_\X\big(F(x, u),\, F(x, u')\big)
\le
K_u^F\, \|u - u'\|,
\end{gathered}
\end{equation}
for some constants $K_x^F, K_u^F \ge 0$.

Since the system includes stochastic process uncertainty whose covariance may depend on the state and the control input, we impose two mild regularity assumptions to ensure that the noise term behaves in a Lipschitz manner.
\begin{assumption}[Lipschitz continuity of process uncertainty]
\label{as:lip-noise}
The process uncertainty is assumed to be Lipschitz continuous with Wasserstein distance. Formally, $\exists K_x^Q, K_u^Q \ge 0$ s.t. $\forall x, x' \in \X, \forall u, u' \in \U$:
\begin{equation}
\begin{gathered}
\label{eq:lip-noise-u}
W_2\bigl(\mathcal{N}\big(0, Q(x, u)\big), \mathcal{N}\big(0, Q(x', u)\big)\bigr)
\le
K_x^Q \, D_\X(x, x'), 
\\
W_2\bigl(\mathcal{N}\big(0, Q(x, u)\big), \mathcal{N}\big(0, Q(x, u')\big)\bigr)
\le
K_u^Q \, \|u-u'\|.
\end{gathered}
\end{equation}
This assumption ensures that changing the state or control does not change the process covariance arbitrarily fast.
\end{assumption}

\begin{assumption}[Lipschitz continuity of noise injection]
\label{as:lip-injection}
The system in \autoref{eq:lip-system-F} preserves Lipschitz continuity with injected noise.
Specifically, $\exists K_x^{F\eta}, K_u^{F\eta}, K_\eta^{F\eta} \ge 0$ such that $\forall x,x' \in \X$, $u \in \U$, and noise realizations $\eta, \eta'$:
\begin{equation}
\label{eq:lip-system-F2}
\begin{gathered}
D_\X\big(F(x,u)\diamond\eta, F(x',u)\diamond\eta\big)
\le
K_x^{F\eta} D_\X(x,x'),
\\
D_\X\big(F(x,u)\diamond\eta, F(x,u')\diamond\eta\big)
\le
K_u^{F\eta} \|u - u'\|,
\\
D_\X\big(F(x,u)\diamond\eta, F(x,u)\diamond\eta'\big)
\le
K_\eta^{F\eta} D_\X(\eta, \eta').
\end{gathered}
\end{equation}
This assumption guarantees that variations in the injected noise lead to proportionally bounded deviations in the resulting state.
\end{assumption}

Under these additional assumptions, we can establish that the belief dynamics and the associated cost functions inherit Lipschitz continuity with respect to the Wasserstein metric.
\begin{proposition}[]
\label{pr:belief-lip-dyn}
Consider a system with Gaussian noise that satisfies the assumptions in \autoref{eq:lip-system-F}, \autoref{eq:lip-noise-u} and \autoref{eq:lip-system-F2}. Then $\exists K_x^{F_b}, K_u^{F_b} \ge 0$ s.t. $\forall b, b' \in \B$ and $u, u' \in \U$:
\begin{equation}
\begin{gathered}
W_2\big( F_b(b, u), F_b(b', u) \big)
\le
K_x^{F_b} \, W_2(b, b'),\\
W_2\big( F_b(b, u), F_b(b, u') \big)
\le
K_u^{F_b} \, \|u - u'\|.
\end{gathered}
\end{equation}
\end{proposition}

\begin{proposition}[]
\label{pr:belief-lip-cost}
Consider a system with Gaussian noise that satisfies assumptions in \autoref{eq:lip-system-F}, \autoref{eq:lip-noise-u} and \autoref{eq:lip-system-F2}. Then $\exists K_x^\lengthb, K_u^\lengthb, K_x^\terminalb \ge 0$ s.t. $\forall b, b' \in \B$ and $u, u' \in \U$:
\begin{equation}
\begin{gathered}
| \lengthb\big(b, F_b(b, u)\big) - \lengthb\big(b', F_b(b', u)\big) |
\le
K_x^\lengthb \, W_2(b, b'),\\
| \lengthb\big(b, F_b(b, u)\big) - \lengthb\big(b, F_b(b, u')\big) |
\le
K_u^\lengthb \, \|u - u'\|,\\
| \terminalb(b) - \terminalb(b') |
\le
K_x^\terminalb \, W_2(b, b'),
\end{gathered}
\end{equation}
if running cost function $\lengthb(\cdot,\cdot)$ is the Wasserstein distance $W_2(\cdot,\cdot)$ and terminal cost $\terminalb(\cdot)$ is the Wasserstein distance between the terminal belief and some constant belief $b_c$.
\end{proposition}

The detailed proofs of~\autoref{pr:belief-lip-dyn} and~\autoref{pr:belief-lip-cost} are provided in~\appref{app:aorrt-belief-lip-dyn-proof} and~\appref{app:aorrt-belief-lip-cost-proof}, respectively.
These propositions imply that, as long as the underlying system dynamics $F$ satisfy the Lipschitz continuity conditions in~\autoref{eq:lip-system-F} (\autoref{as:lip-system}), the corresponding belief dynamics and cost functions remain Lipschitz continuous in the belief space when equipped with the Wasserstein metric. 
As a result, AO-RRT with terminal costs preserves asymptotic optimality when operating in the belief state-cost space.

\subsection{Belief Formulation for Common Systems}
\label{sec:aorrt-belief-prop}
We next illustrate how the proposed belief-space formulation is instantiated for two common classes of systems: 1) linear or linearized systems in Euclidean space, and 2) rigid-body systems evolving on matrix Lie groups.

\subsubsection{Linear or Linearized System}

Consider a linear or linearized system with additive Gaussian noise in Euclidean space:
\begin{equation}
\label{eq:linear-dyn-general}
\begin{gathered}
x_{i+1} = F(x_i, u_i) \diamond \eta_i = A x_i + B u_i + \eta_i, 
\\
\eta_i \sim \mathcal{N}\big(0, Q(x_i, u_i)\big),
\end{gathered}
\end{equation}
where $A$ and $B$ are constant matrices.
We represent the belief state as a Gaussian distribution $b_i = \mathcal{N}(x_i, P_i)$.
Under open-loop propagation, the belief dynamics $F_b$ updates the mean and covariance as~\cite{ccrrt}:
\begin{equation}
\label{eq:linear-belief-update}
\begin{aligned}
x_{i+1} &= A x_i + B u_i,\\
P_{i+1} &= A P_i A^\top + Q(x_i, u_i).
\end{aligned}
\end{equation}

Since the state space is Euclidean, the Wasserstein distance between Gaussian beliefs admits a closed form directly.
Therefore, the running cost $\lengthb(\cdot,\cdot)$ and terminal cost $\terminalb(\cdot)$ introduced earlier can be evaluated naturally in this setting using the Gaussian belief representation.


\subsubsection{Rigid-body System on Lie Groups}

Rigid-body configurations often evolve on a matrix Lie group $G$, such as $\SETwo$ or $\SEThree$. 
Since $G$ is not a vector space, Gaussian uncertainty is more naturally defined in its associated Lie algebra $\mathfrak{g}$, which is the tangent space at the identity and admits a linear structure. 
The exponential and logarithm maps, $\Exp : \mathfrak{g} \to G$ and $\Log : G \to \mathfrak{g}$, provide the local correspondence between the manifold and tangent space~\cite{lie_math}.
A state $x$ near a reference $x_{\text{ref}} \in G$ can be expressed as $x = x_{\text{ref}} \, \Exp(\xi)$ with $\xi \in \mathfrak{g}$.

We model the dynamics with multiplicative noise as:
\begin{equation}
\label{eq:lie-dyn-general}
\begin{gathered}
x_{i+1} = F(x_i, u_i) \diamond \eta_i = F(x_i, u_i)\,\Exp(\eta_i),
\\
\eta_i \sim \mathcal{N}\big(0, Q(x_i, u_i)\big),
\end{gathered}
\end{equation}
where $F : G \times \U \to G$ is the nominal rigid-body update. Gaussian noise $\eta_i$ is defined in the tangent space. 
A belief state is represented as a Lie-Gaussian belief $b = \mathcal{N}_L(x, P)$, where $x$ is the nominal pose and $P$ is the covariance in the tangent space. A random sample is given by:
\begin{equation}
X \sim \mathcal{N}_L(x, P) = x \,\Exp (\mathcal{N}(0, P)).
\end{equation}
Under ~\autoref{eq:lie-dyn-general}, the belief dynamics $F_b$ evolves according to~\cite{lie_math}:
\begin{equation}
\label{eq:lie-belief-update}
\begin{gathered}
x_{i+1} = F(x_i, u_i), 
\\
\xi_i = \Log(x_i^{-1}x_{i+1}), 
\\
P_{i+1} 
= \Adj_{\Exp(-\xi_i)}\,P_i\,\Adj_{\Exp(-\xi_i)}^\top
+ 
Q(x_i, u_i),
\end{gathered}
\end{equation}
where adjoint action $\Adj$ transports covariances between tangent spaces.
%


%

A common global metric for a Lie-group system is the double geodesic distance~\cite{distance_3d}:
\begin{equation}
\label{eq:dual-geodesic}
D_G(x, x')
=
\|p - p'\|
+
w_r \, \| \Log(R^{-1} R') \|,
\end{equation}
where $p,p'$ and $R,R'$ denote the positional and rotational components of $x$ and $x'$, respectively, and $w_r$ is a weight.
However, since it is not induced by a global linear structure, the corresponding Wasserstein distance does not admit a simple closed form.
To obtain a tractable Gaussian Wasserstein distance on a Lie group between beliefs $b=\mathcal{N}_L(x, P)$ and $b'=\mathcal{N}_L(x', P')$, we linearize both beliefs in a common tangent space. A natural choice is to use one of the means $x$ as the reference point.
In this frame, the relative means and covariances are:
\begin{equation}
\begin{gathered}
\hat x = I,
\quad
\hat x' = x^{-1}x',
\\
\hat P = P,
\quad
\hat P' = \Adj_{x^{-1}x'}\,P'\,\Adj_{x^{-1}x'}^\top.
\end{gathered}
\end{equation}
This yields a local metric:
\begin{equation}
\label{eq:lie-l2-approx}
D_G(x,x') = \| \Log(\hat x) - \Log(\hat x') \| = \| \Log(x^{-1} x') \|,
\end{equation}
which matches the true geodesic up to second order near the reference state $x$. Under this linearization, the two beliefs become Gaussian in the vector space $\mathfrak{g}$, namely $\mathcal{N}(\hat x, \hat P)$ and $\mathcal{N}(\hat x', \hat P')$.
We then approximate the Wasserstein distance by the closed-form Gaussian expression in the same tangent space:
\begin{equation}
\label{eq:lie-w2-approx}
W_2^2(b,b')
=
\|\Log(\hat x')\|^2
+
\mathrm{Tr} \big(
P + \hat P' - 2(\hat P'^{\frac12} P\, \hat P'^{\frac12})^{\frac12}
\big).
\end{equation}

In summary, rigid-body systems on $G$ admit intrinsic belief propagation via \autoref{eq:lie-belief-update} and a local Gaussian Wasserstein metric approximation \autoref{eq:lie-w2-approx}. 
With the required Lipschitz conditions satisfied, these systems can fit directly into the AO-RRT belief-space framework.

\section{Goal-reaching Success Lower Bound}
\label{sec:aorrt-bound}
As discussed earlier, in belief-space kinodynamic planning, the system state is modeled as a random variable rather than a deterministic quantity.
This allows the planner to reason explicitly about uncertainty, enabling chance-constrained formulations~\cite{ccrrt, belief_tree} that reject states with high collision risk and terminal states that fail to reach the goal region with sufficiently high probability:
\begin{equation}
\label{eq:chance_constrain}
\begin{gathered}
    \Pr\!\left( x \in \Xfree \right) < p_\free, \quad
    \Pr\!\left( x_\text{end} \in \G \right) < p_\goal,
\end{gathered}
\end{equation}
where $p_\free$ and $p_\goal$ denote the minimum required probabilities of collision-free validity and goal reaching, respectively.

In practice, however, identifying appropriate threshold values across different environments and tasks is not trivial.
For collision avoidance, a conservative threshold (high $p_\free$) is often acceptable, since many alternative actions may still remain available during sampling.
By contrast, $p_\goal$ can be difficult to specify a priori for goal reaching. If it is set too high, the planner may fail to find a feasible solution, whereas if it is too low, it may accept solutions with weak goal-reaching probability. In fact, as shown in~\autoref{sec:experiments}, optimizing only the running cost may lead the planner to favor paths with low goal-reaching success.
Ideally, rather than enforcing a hard threshold, the planner should use the available time budget to continuously improve the goal-reaching probability. 

In this work, instead of imposing a hard goal-reaching threshold, we derive a lower bound on the probability of reaching the goal region using the Wasserstein distance between the terminal belief state and the goal. 
This bound can be incorporated as a terminal cost in the optimization objective, allowing AO-RRT to favor trajectories with higher goal-reaching success without requiring manual threshold tuning.

\subsection{Lower Bound on Goal-reaching Probability}
\label{sec:aorrt-belief-success}
To incorporate goal-reaching probability into the planning objective, we relate it to a geometric quantity in belief space.
Specifically, the probability of reaching a ball-shaped or ellipsoidal goal region at the goal state $g$ can be lower-bounded using the Wasserstein distance between the terminal belief state $b$ and the Dirac measure $\nu_g$ at $g$.
Intuitively, the Dirac measure $\nu_g$ represents a deterministic belief concentrated entirely at the single state $g$.

We now state the formal theorem and provide its proof.
\begin{theorem}[Goal-reaching Lower Bound]
\label{thm:lower-bound}
Let $b$ be a probability measure on $\X$ with finite second moment, and let $\nu_g$ denote the Dirac measure at a goal state $g \in \X$. For any radius $r > 0$, define the goal region as the ball $\G = B_r(g)$. If $X \sim b$, then the probability of reaching the goal region $\G$ satisfies
\begin{equation}
\label{eq:w2-success-bound}
\Pr\left(X \in \G\right)
\ge
1 - \frac{W_2^2 \left(b, \nu_g\right)}{r^2}.
\end{equation}
\end{theorem}

\begin{proof}
Let $X \sim b$ and define the nonnegative random variable $J := \|X - g\|^2$. By the definition of $\G$,
\begin{equation}
\begin{aligned}
\Pr\left( X \in \G \right)
&=
\Pr\left( \|X - g\| \le r \right)\\
&=
1 - \Pr\left( \|X - g\| > r \right)\\
&=
1 - \Pr\left( J > r^2 \right).
\end{aligned}
\end{equation}
Applying Markov's inequality to $J$ gives:
\begin{equation}
\label{eq:markov-ineq}
\Pr\left( J > r^2 \right)
\le
\Pr\left( J \ge r^2 \right)
\le
\frac{\mathbb{E}[J]}{r^2}
=
\frac{\mathbb{E}[\|X - g\|^2]}{r^2}.
\end{equation}

We now relate $\mathbb{E}[\|X - g\|^2]$ to the Wasserstein distance between $b$ and $\nu_g$:
\begin{equation}
\label{eq:w2-def}
W_2^2(b,\nu_g)
=
\inf_{\gamma \in \Gamma(b, \nu_g)}
\int \|x - v\|^2 \, d\gamma(x, v).
\end{equation}
Since $\nu_g$ is a Dirac measure at $g$, we have $v = g$ almost surely under every coupling. Thus, the integral in \autoref{eq:w2-def} reduces to:
\begin{equation}
\begin{aligned}
\label{eq:w2-simplify}
\inf_{\gamma \in \Gamma(b, \nu_g)}
\int \|x - v\|^2 \, d\gamma(x, v)
&= \int \|x - g\|^2 \, db(x)\\
&= \mathbb{E}[\|X-g\|^2].
\end{aligned}
\end{equation}

Therefore,
\begin{equation}
\begin{aligned}
\Pr\left( X \in \G \right)
&= 1 - \Pr\left( J > r^2 \right)
\\&\ge 1 - \frac{\mathbb{E}[\|X - g\|^2]}{r^2}
\\&= 1 - \frac{W_2^2\left(b, \nu_g \right)}{r^2}.
\end{aligned}
\end{equation}
\end{proof}

If the belief state $b$ is Gaussian, $b = \mathcal{N}(x,P)$, then its squared Wasserstein distance to the Dirac measure $\nu_g$ admits the closed form:
\begin{equation}
\label{eq:terminal_gaussian}
W_2^2( b, \nu_g )
= \|x - g\|^2 + \mathrm{tr}(P).
\end{equation}
The following corollary is immediate from \autoref{thm:lower-bound}.
\begin{corollary}[]
\label{cor:w2-gaussian}
If the belief state $b$ is Gaussian, i.e., $X \sim \mathcal{N}(x,P)$, and $\G = B_r(g)$, the probability of reaching the goal region satisfies
\begin{equation}
\label{eq:gaussian-success-bound}
\Pr \left(X \in \G \right)
\ge
1 -
\frac{\|x - g\|^2 + \mathrm{tr}(P)}{r^2}.
\end{equation}
\end{corollary}

Note that~\autoref{eq:gaussian-success-bound} contains the distance between the belief mean and the goal, $\|x - g\|$. This motivates using only the state distance to the goal, $\|x - g\|$, without the uncertainty term, as a simple surrogate in practice.
We evaluate this empirically in~\autoref{sec:experiments}.
However, we emphasize that $\|x - g\|$ alone does not in general induce a valid lower bound on $\Pr \left(X \in \G \right)$, and should only be viewed as a computationally convenient heuristic.
A belief with mean close to $g$ but large covariance may still assign little probability mass to the goal region.

A parallel result applies when the goal region is an ellipsoid.
Let $E$ be a symmetric positive definite matrix and define the squared $E$-norm:
\begin{equation}
\|q\|_E^2 := q^\top E^{-1} q.
\end{equation}
For $r > 0$, we define an ellipsoidal goal region as:
\begin{equation}
\G_E
=
\{\, (x - g)^\top E^{-1} (x-g) \le r^2 \,\}
=
\{\, \|x - g\|_E \le r \,\}.
\end{equation}
We then define the weighted Wasserstein distance to $\nu_g$:
\begin{equation}
\label{eq:w2q-def}
W_{2,E}^2(b,\nu_g)
:=
\inf_{\gamma \in \Gamma(b,\nu_g)}
\int \|x - v\|_E^2 \, d\gamma(x,v).
\end{equation}
The next corollary follows the same argument as \autoref{thm:lower-bound}, replacing the Euclidean norm by the weighted norm $\|\cdot\|_E$.
\begin{corollary}
\label{cor:w2-ellipsoid}
If $X \sim b$, the probability of $X$ lying in an ellipsoidal goal region $\G_E$ is lower-bounded by:
\begin{equation}
\label{eq:w2-ellipsoid-bound}
\Pr \left(X \in \G_E\right)
\ge
1 - \frac{W_{2,E}^2(b,\nu_g)}{r^2}.
\end{equation}
\end{corollary}

\subsection{Goal-reaching Success Lower Bound as Terminal Cost}
\label{sec:aorrt-terminal-cost}
Since Wasserstein distance defines a Lipschitz-continuous terminal cost in $\B$ (\autoref{pr:belief-lip-cost}), the lower bound in \autoref{thm:lower-bound} can be incorporated directly into the planning objective.
Rather than enforcing a hard goal-reaching threshold, we use the Wasserstein distance from the final belief $b_m$ to the Dirac measure of the goal $\nu_g$ as a terminal cost. For simplicity, we call it Wasserstein distance to the goal:
\begin{equation}
\psi(b_m) = w_g \, W_2(b_m, \nu_g),
\end{equation}
where $w_g$ is a weight.
By \autoref{thm:lower-bound}, reducing $W_2(b_m,\nu_g)$ directly improves a lower bound on the probability of reaching the goal region.
For Gaussian beliefs, \autoref{cor:w2-gaussian} shows that this terminal cost jointly penalizes mean deviation from the goal and terminal uncertainty.
Combined with the running cost, this yields the belief-space objective:
\begin{equation}
\cost_\B(\btraj, \utraj)
=
\sum_{i=0}^{m-1} W_2(b_i, F_b(b_i, u_i))
+ w_g \, W_2(b_m, \nu_g).
\end{equation}
We provide two examples to highlight the role of the terminal cost in belief-space planning.

\begin{figure}[!ht]
    \centering
    \includegraphics[width=0.9\linewidth]{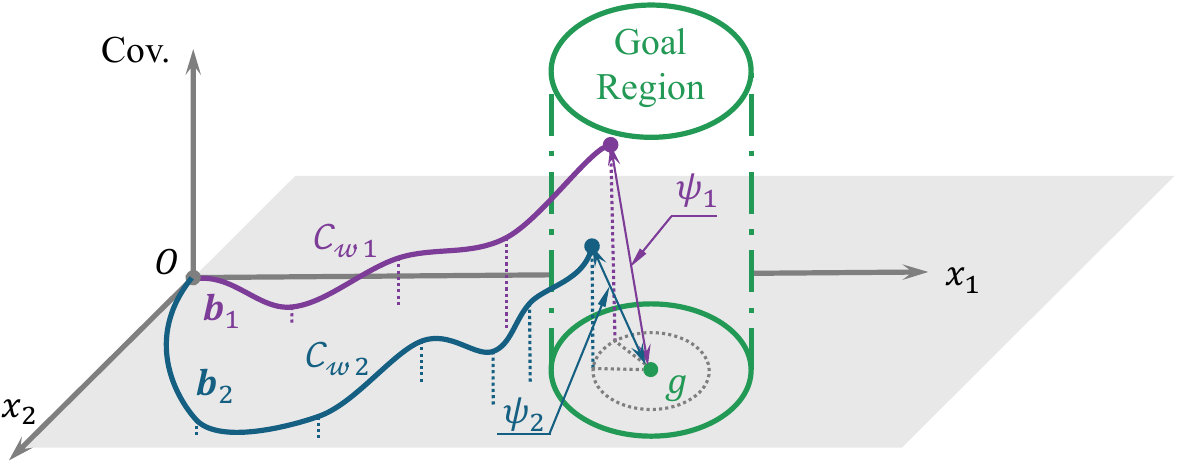}
    \caption{Example where running cost alone is insufficient. Although $\btraj_2$ has a larger running cost due to its longest state-space path, it reaches the goal with a smaller Wasserstein distance to $\nu_g$, giving it the highest lower bound on goal-reaching probability.}
    \label{fig:terminal1}
\vspace{-2ex}
\end{figure}

\begin{figure}[!ht]
    \centering
    \includegraphics[width=0.9\linewidth]{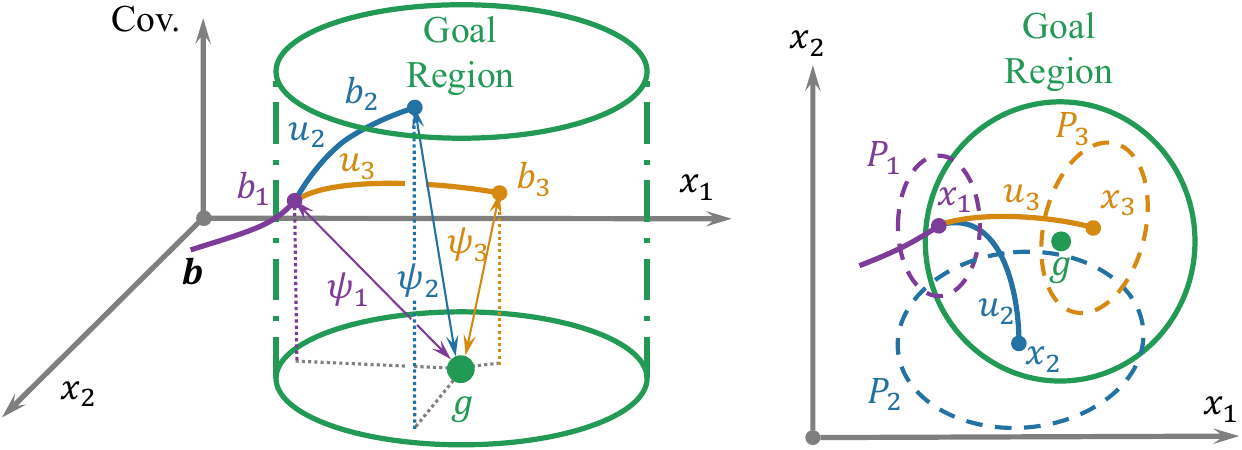}
    \caption{Example where terminal cost is necessary. Although extending the belief $b_1$ always increases accumulated running cost, a successor such as $b_3$ can still reduce the terminal Wasserstein distance to $\nu_g$ and improve the lower bound on goal-reaching probability.}
    \label{fig:terminal2}
\vspace{-2ex}
\end{figure}

The first example is shown in \autoref{fig:terminal1}.
The vertical axis represents covariance magnitude, quantified by its trace, while the horizontal plane is the state space.
With process uncertainty, each propagation step increases uncertainty.
Consider two trajectories that start from the same belief and end at beliefs whose means are equally distant from $g$.
Although $\btraj_2$ has a larger running cost than that of $\btraj_1$, it ends with a smaller terminal covariance.
Therefore, $\btraj_2$ has the highest lower bound on goal-reaching probability (\autoref{eq:gaussian-success-bound}).
This shows that running cost may loosely correlate with terminal quality, but it does not optimize it directly and can rank trajectories incorrectly.

A second example is shown in \autoref{fig:terminal2}.
Suppose the tree reaches a belief $b_1 \sim \mathcal{N}(x_1, P_1)$ whose mean is already inside the goal region.
In standard AO-RRT with only running costs, extending from $b_1$ is unattractive because every successor increases accumulated cost.
However, terminal quality can still improve.
A noisy action to $b_2 \sim \mathcal{N}(x_2,P_2)$ worsens the terminal Wasserstein distance, even though $x_2$ is closer to $g$ in state space.
But a precise action to $b_3 \sim \mathcal{N}(x_3,P_3)$ moves the mean closer to $g$ while adding little uncertainty, giving $b_3$ a smaller terminal Wasserstein distance than $b_1$.
Thus, although both successors have a higher running cost, $b_3$ is better in terms of goal-reaching probability.
Without a terminal term, AO-RRT would not favor such beneficial extensions.

Together, these examples show that the Wasserstein running cost in belief space is only a loose surrogate for goal-reaching probability and may even conflict with it.
This motivates augmenting the running-cost objective with an explicit terminal Wasserstein cost.
The running cost remains important for guiding the search toward efficient trajectories that avoid unnecessary drift in the mean or covariance.
The terminal cost complements it by improving a lower bound on the goal-reaching probability and correctly ranking terminal beliefs.
%


\section{Learning Dynamics and Uncertainty}
\label{sec:modeling}
For many robotic systems, the dynamics and process uncertainty are difficult to model in closed form~\cite{ltamp, poking}. Even when an approximate analytical model is available, its parameters may still be difficult to identify accurately~\cite{activepusher, panda_param2019}.
In such settings, we train a probabilistic forward model from interaction data and use it to propagate beliefs during planning.

Rather than directly predicting the next state, we learn a local stochastic transition model.
Specifically, let $\oplus$ denote a system-dependent composition operator such that
\begin{equation}
\label{eq:learn-transition}
x_{i+1} = x_i \oplus \Delta x_i,
\end{equation}
where $\Delta x_i$ denotes the local transition variable associated with applying control $u_i$ at state $x_i$.
We model this local transition as a Gaussian random variable:
\begin{equation}
    \Delta X_i \sim \mathcal{N}(\Delta x_i, Q_i),
    \label{eq:local-transition-gaussian}
\end{equation}
where $\Delta x_i$ is the mean local transition and $Q_i$ is the associated process covariance.
In this work, for simplicity, we use a diagonal covariance model for learning process uncertainty:
\begin{equation}
    Q_i = \mathrm{Diag}(\sigma_i^2),
    \label{eq:diag-qi}
\end{equation}
where $\sigma_i^2$ denotes the variance vector.
Thus, the learned model aims to predict both the mean local transition $\Delta x_i$ and its associated process covariance parameterized by $\sigma_i^2$, which are then combined with the current belief $(x_i, P_i)$ to obtain the next belief state $(x_{i+1}, P_{i+1})$.

\begin{figure}[!t]
    \centering
    \includegraphics[width=\linewidth]{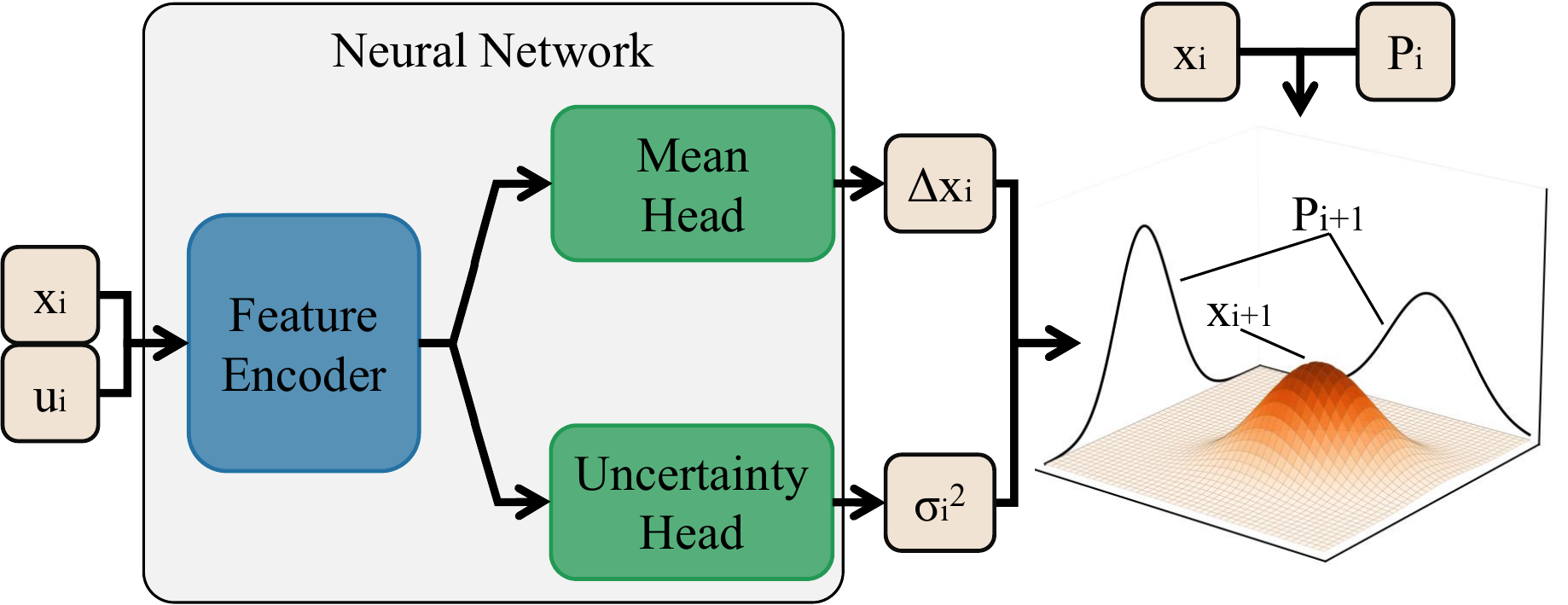}
    \caption{Neural network model for learning dynamics and uncertainty. The model predicts a mean local transition $\Delta x_i$ and a diagonal process covariance parameterized by $\sigma_i^2$. Combined with the current belief $(x_i, P_i)$, these predictions are used to propagate and obtain the next belief $(x_{i+1}, P_{i+1})$.}
    \label{fig:model}
\vspace{-2ex}
\end{figure}

The specific form of $\oplus$ and the corresponding covariance propagation depends on the geometric of the state space.
For linear Euclidean systems, as in~\autoref{eq:linear-belief-update}, $\oplus$ reduces to vector addition, so the state update becomes
\begin{equation}
    x_{i+1} = x_i + \Delta x_i.
\end{equation}
Its covariance is then propagated with first-order linearization:
\begin{equation}
\begin{gathered}
    P_{i+1} = A_i P_i A_i^\top + Q_i, \quad
    A_i = I + \frac{\partial \Delta x_i}{\partial x_i},
\end{gathered}
\label{eq:learned-cov-linear}
\end{equation}
where $I$ is the identity matrix and $A_i$ is the Jacobian of the mean transition with respect to the current state.
For Lie-group systems, $\oplus$ corresponds to composition through the exponential map, so the mean transition becomes
\begin{equation}
\label{eq:learned-mean-lie}
    x_{i+1} = x_i \, \Exp(\Delta x_i).
\end{equation}
In this case, the covariance is propagated in the tangent space:
\begin{equation}
\label{eq:learned-cov-lie}
    P_{i+1}
    =
    \Adj_{\Exp(-\Delta x_i)}
    \, P_i \,
    \Adj_{\Exp(-\Delta x_i)}^\top
    +
    Q_i,
\end{equation}
which is consistent with the belief update in \autoref{eq:lie-belief-update}.

The overall neural network structure is shown in \autoref{fig:model}.
Given the current state $x_i$ and control $u_i$, the network first extracts a shared feature representation and then branches into a mean head and an uncertainty head.
The mean head predicts the mean local transition $\Delta x_i$, while the uncertainty head provides the corresponding variance $\sigma_i^2$.
For numerical stability, the uncertainty head predicts the log-variance in practice.
In this work, the forward model is implemented as a simple multilayer perceptron (MLP).

The model is trained in a supervised manner using tuples $(x_i, u_i, x_{i+1})$.
Several probabilistic regression losses can be used to jointly learn dynamics and uncertainty, such as evidential regression~\cite{evidential} and $\beta$-NLL~\cite{beta-nll}.
In this work, we use the standard Gaussian negative log-likelihood (NLL) loss:
\begin{equation} 
\label{eq:nll-loss} 
\mathcal{L}_{\mathrm{NLL}} 
= \frac{1}{2} \left( 
    \frac{(\Delta x_i^{\mathrm{gt}} - \Delta x_i)^2}{\sigma_i^2} 
    + \log \sigma_i^2
\right), 
\end{equation}
where $\Delta x_i^{\mathrm{gt}}$ denotes the ground-truth local transition.
This objective encourages accurate transition prediction while calibrating the predicted uncertainty according to the observed variability in the data.
The learned uncertainty primarily captures aleatoric effects arising from stochasticity, model mismatch, and unmodeled interactions.

\section{Experiments}
\label{sec:experiments}

In this section, we evaluate three aspects of the proposed framework:
(1) whether the terminal cost allows the planner to encode goal desirability,
(2) whether it improves goal-reaching probability under uncertainty, and
(3) whether learned belief dynamics can be integrated into AO planners.
The experiments are conducted on three tasks with increasing complexity: Flappy Bird, Car Parking, and Planar Pushing.
The last two tasks are run in Mujoco simulation~\cite{mujoco}, while the Planar Pushing task is also evaluated in the real world.
All planning and learning experiments are run on a computer with an i7-14700KF CPU, 32GB of RAM and a RTX 4070 Ti Super GPU.

\subsection{Flappy Bird Task}
\label{sec:experiment-flappy}

We begin with a Flappy Bird benchmark adapted from~\cite{aox} to isolate the effect of terminal cost in a deterministic setting.
The bird moves from left to right with constant horizontal speed and can optionally apply an upward thrust.
Its state is defined as
$x = [p_x,\ p_y,\ v_y]^\top$,
where $p_x$ and $p_y$ denote the horizontal and vertical positions in pixels (px), and $v_y$ denotes the vertical velocity.
The control is binary, $u \in \{0,1\}$, indicating whether an upward thrust is applied.
The frame-based dynamics are
\begin{equation}
\label{eq:flappy_dynamics}
\begin{bmatrix}
p_{x,i+1} \\
p_{y,i+1} \\
v_{y,i+1}
\end{bmatrix}
=
\begin{bmatrix}
p_{x,i} \\
p_{y,i} \\
v_{y,i}
\end{bmatrix}
+
\begin{bmatrix}
v_x \\
v_{y,i} \\
g_p + u_i f_p
\end{bmatrix}.
\end{equation}
where $v_x = 5$ px/frame, $g_p = -1.0$ px/frame$^2$, and $f_p = 9$ px/frame$^2$.
The screen size is $720\times480$ pixels, and $v_y$ is clipped to $[-10,10]$ px/frame.

A trajectory is feasible if it stays within the screen, avoids all pillars, and reaches the goal region after the last pillar, as shown in \autoref{fig:flappy_demo}.
We define the running cost with accumulated obstacle clearance, i.e., the distance from the bird to the nearest pillar, so that safer trajectories are preferred.
To encode goal desirability, we define a terminal cost that penalizes deviation from the center of the final gap.
This corresponds to the setting where the next planning starts from the current terminal state, so ending near the center provides a more desired initial condition.

We compare \method with vanilla AO-RRT~\cite{aox} and SST~\cite{sst}, without terminal costs. All methods  are implemented in OMPL~\cite{ompl}.
We denote the three methods as Base (AO-RRT), Base (SST), and \method (AO-RRT, $w_g$), where $w_g$ is the terminal weight. Note that SST is not theoretically justified for incorporating a terminal cost.
We generate 20 random pillar layouts and run each planner 5 times per environment.
The results are averaged across environments, with standard deviations computed over repeated runs.

As shown in \autoref{fig:flappy_cost}, all planners improve solution quality as more planning time is provided.
However, \method consistently returns trajectories whose terminal states lie closer to the center of the final gap.
This comes with only a modest increase in running cost, which reflects the intended trade-off. 

\begin{figure}[!t]
    \centering
    \includegraphics[width=\linewidth]{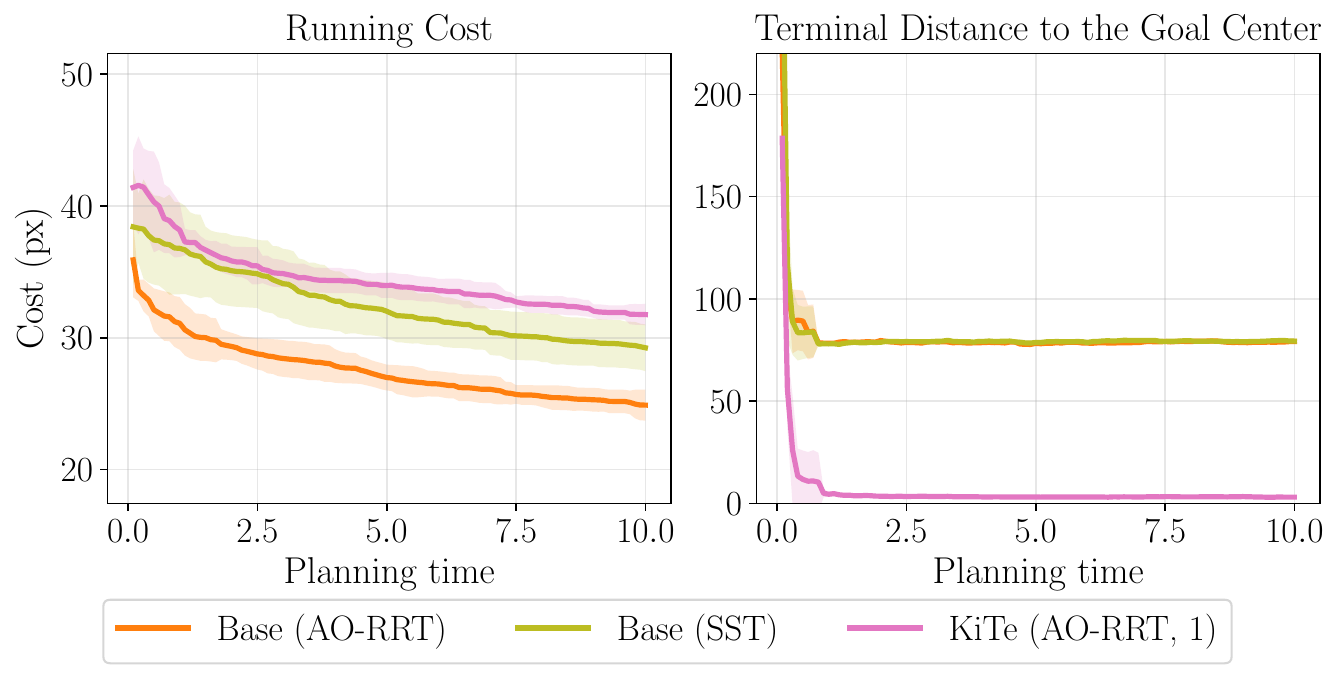}
    \vspace{-4ex}
    \caption{Comparison of running cost and terminal distance to the goal center. \method reduces terminal distance to the goal center while incurring a modest increase in running cost.}
    \label{fig:flappy_cost}
\vspace{-3ex}
\end{figure}
\begin{figure}[!t]
    \centering
    \includegraphics[width=\linewidth]{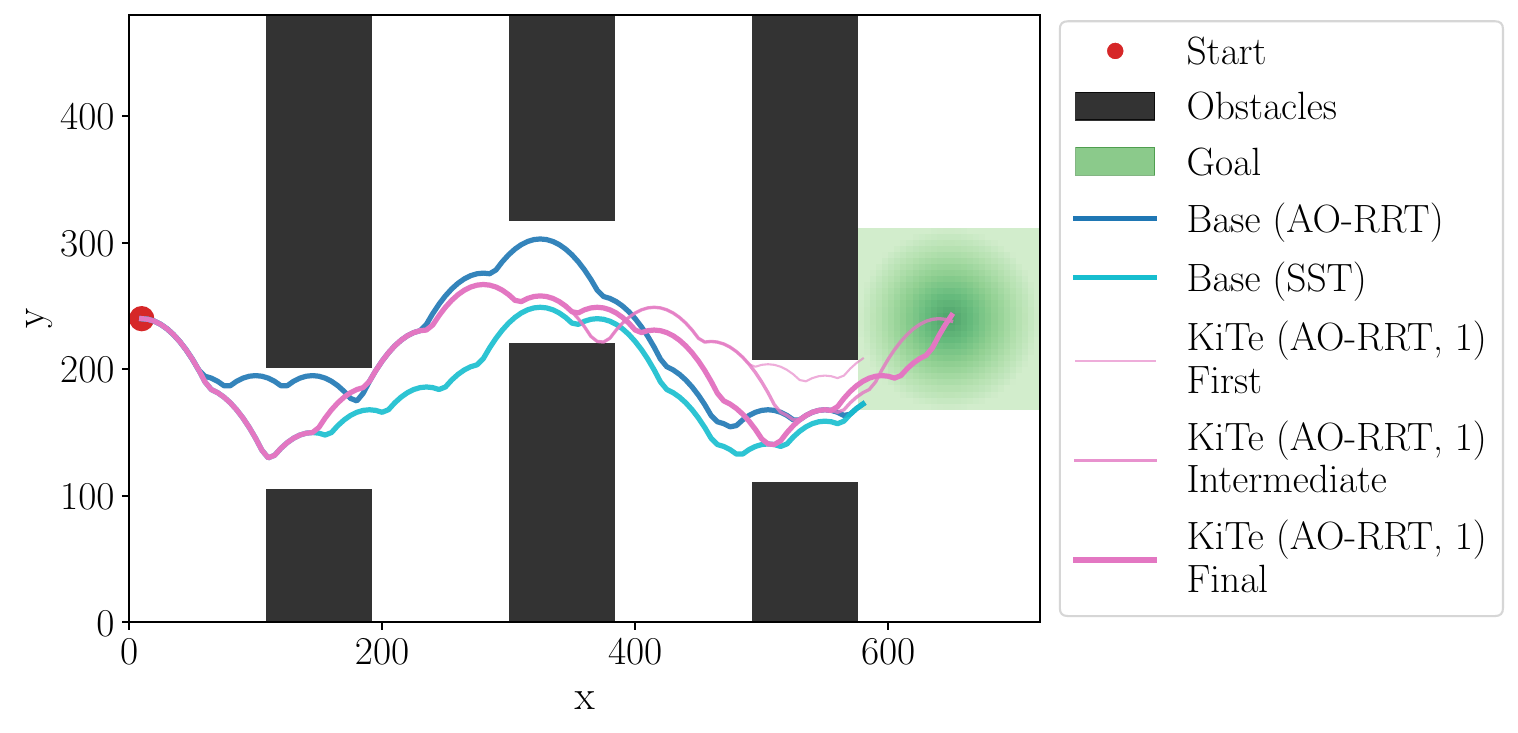}
    \vspace{-4ex}
    \caption{Representative planning process in the Flappy Bird task. \method first finds a feasible trajectory (\method First), then improves the terminal state toward the center (\method Intermediate), and eventually achieves both low running and terminal costs (\method Final). In contrast, other methods have no incentive to improve further once the goal region is reached.}
    \label{fig:flappy_demo}
\vspace{-2ex}
\end{figure}

\autoref{fig:flappy_demo} shows a representative planning process.
At early iterations, \method finds a feasible path that merely enters the goal region.
Given more time, it identifies alternative trajectories that terminate closer to the center of the gap.
Eventually, it finds trajectories that jointly achieve low running cost and low terminal cost.
In contrast, when only running cost is optimized, extending farther within the set is usually undesired because any additional motion increases the accumulated path cost.
The terminal term changes this behavior by making improvement inside the goal region explicitly beneficial.

One alternative to achieve goal desirability is to shrink the goal region towards the desired state.
However, doing so makes the feasible planning problem harder under a fixed time budget.
In contrast, the terminal-cost formulation preserves the original feasible region and allows the planner to progressively improve terminal quality after feasibility has already been achieved.

Overall, this experiment shows that incorporating terminal cost allows the planner to preserve the original feasible region and progressively improve the terminal state to satisfy the specified goal desirability.
%

\subsection{Car Parking Task}
\label{sec:experiment-car}

We next consider a car parking task in $\SETwo$ with nonholonomic dynamics.
We model the car by the simplified kinematic bicycle model~\cite{car-dynamics} with state
$x = [p_x,\ p_y,\ \theta]^\top$ and control
$u = [u_v,\ u_\phi]^\top$:
\begin{equation}
\label{eq:car_dynamics}
f(x,u)=
\begin{bmatrix}
\dot p_x\\
\dot p_y\\
\dot \theta
\end{bmatrix}
=
\begin{bmatrix}
u_v \cos\theta \\
u_v \sin\theta \\
\dfrac{u_v}{L_{\text{car}}}\tan u_\phi
\end{bmatrix},
\end{equation}
where $(p_x,p_y)$ is the position, $\theta$ is the heading angle,
$u_v$ is the commanded longitudinal velocity, $u_\phi$ is the steering angle, and $L_{\text{car}}$ is the wheelbase of the vehicle.
To obtain belief propagation as in~\autoref{eq:linear-belief-update}, we linearize the dynamics around state $x_i$ to compute the state-transition Jacobian $A_i$:
\begin{equation}
\begin{aligned}
x_{i+1} &= F(x_i, u_i) = x_i + \int_0^{\tau_i} f(x(t), u_i)\, dt, \quad x(0)=x_i,
\\
P_{i+1} &= A_i P_i A_i^\top + Q_i, \quad A_i = \frac{\partial F(x_i, u_i)}{\partial x_i},
\end{aligned}
\end{equation}
where $\tau_i$ is the duration and $Q_i$ is the process uncertainty.
%
%

Uncertainty in this task comes from the mismatch between the bicycle model and the MuJoCo execution dynamics, together with actuation error.
We model this uncertainty by a control-dependent diagonal covariance:
\begin{equation}
\begin{aligned}
\label{eq:car_process_noise}
Q_i
&=
|u_v| \tau_i \, \mathrm{Diag}(\alpha_x,\alpha_y,\alpha_\theta)
\\
&+
|u_v \tan u_\phi| \tau_i \, \mathrm{Diag}(\beta_x,\beta_y,\beta_\theta),
\end{aligned}
\end{equation}
which captures uncertainty induced by translational and turning motion.
%
The coefficients
$\alpha_x,\alpha_y,\alpha_\theta,\beta_x,\beta_y,\beta_\theta$
are estimated by fitting the predicted variance to the empirical transition variance observed in simulation using non-negative least squares.

The objective of this task is to park in one of the feasible goal regions, $g_1$ or $g_2$, while satisfying the forward-facing orientation and avoiding collisions, as illustrated in \autoref{fig:car_demo}.
The front region $g_2$ is treated as more desirable, assuming it is easier to depart later.
We consider both regular state-space planning, which ignores uncertainty, and belief-space planning, which propagates uncertainty explicitly.
The running cost is path length, measured by the $\SETwo$ distance in regular state space (\autoref{eq:lie-l2-approx}) and by the approximated Wasserstein distance in belief space (\autoref{eq:lie-w2-approx}).
For \method, the terminal cost is used both to prefer the desirable parking region and to improve goal-reaching probability.
Specifically, the terminal cost is defined with respect to $g_2$.
We use geometric distance to $g_2$ as the regular-state surrogate and Wasserstein distance to $g_2$ as the belief-space terminal cost.
Accordingly, we report \method (AO-RRT, L2, $w_g$) and \method (AO-RRT, W2, $w_g$), where L2 represents the regular-state metric, W2 denotes the Wasserstein metric, and $w_g$ is the terminal weight.

In this task, we additionally compare against the state-of-the-art belief-space method Gaussian Belief Tree (GBT)~\cite{belief_tree}, which uses AO planners to plan in belief space with Wasserstein distance and chance-constraint formulation.
We denote it as GBT (SST, W2), and GBT (AO-RRT, W2).
All belief-space methods use a chance-constrained threshold $p_\text{free}=0.95$.

The task environment is fixed.
We generate 20 random initial poses, repeat planning 5 times for each problem with 30-second planning budget, and then execute the plans in MuJoCo.
We report averaged planning costs and realized task success rate.
To fairly compare regular state-space and belief-space planners, we post-process each returned plan by propagating the belief dynamics along its action sequence, and then measure and report its cost using the Wasserstein distance.
This yields a unified measure of quality across all methods.

\begin{figure}[!t]
    \centering
    \includegraphics[width=\linewidth]{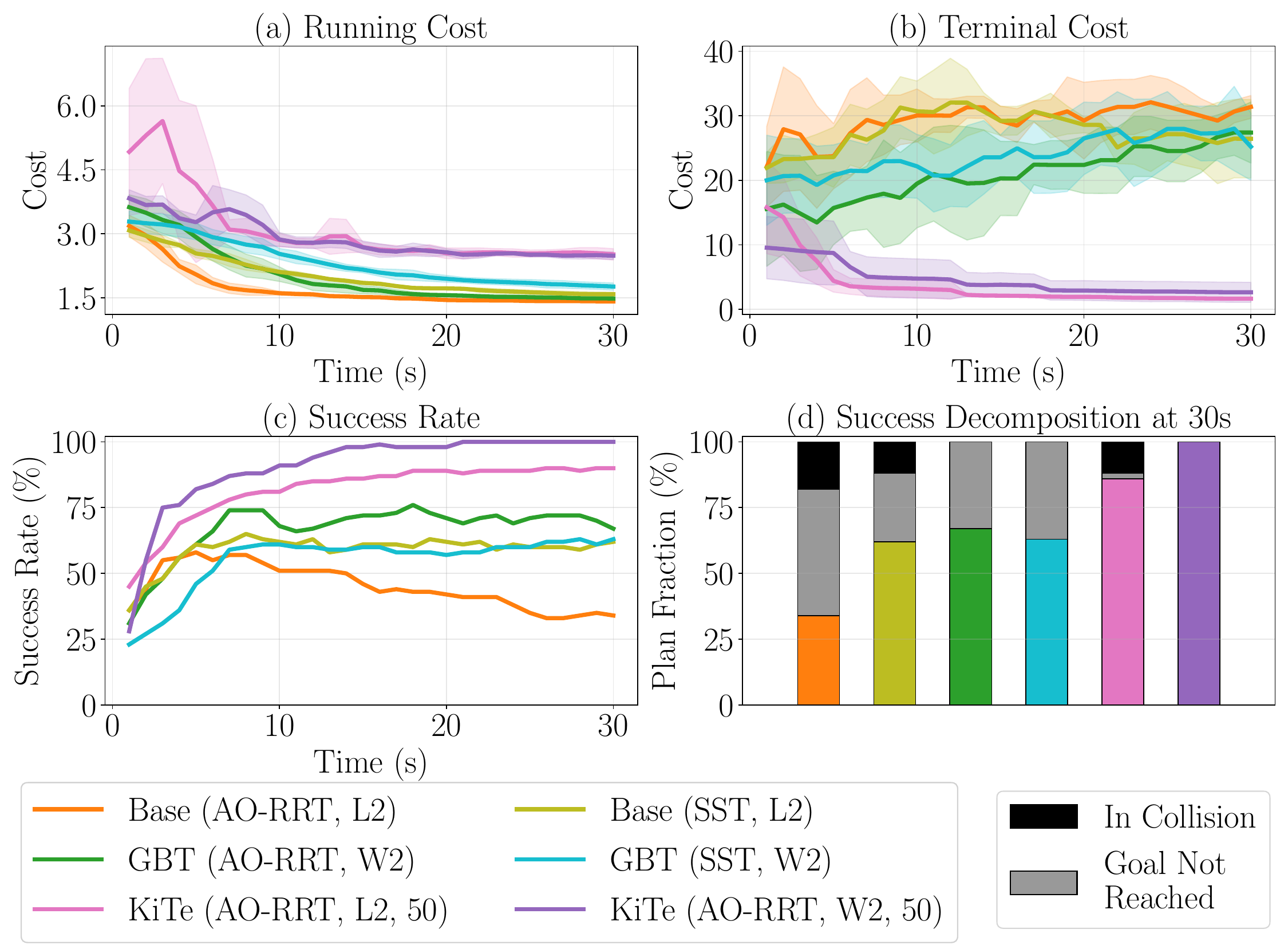}
    \caption{Results for the Car Parking task. 
    (a) Running cost and (b) terminal cost show that \method trades a modest increase in path length for a more favorable terminal state in the desired parking region. 
    (c) Execution success rates and (d) success breakdown show that belief-space planning enhances collision avoidance, while the terminal cost improves goal-reaching performance.}
    \label{fig:car_result}
\vspace{-2ex}
\end{figure}
\begin{figure*}[t]
    \centering
    \includegraphics[width=\linewidth]{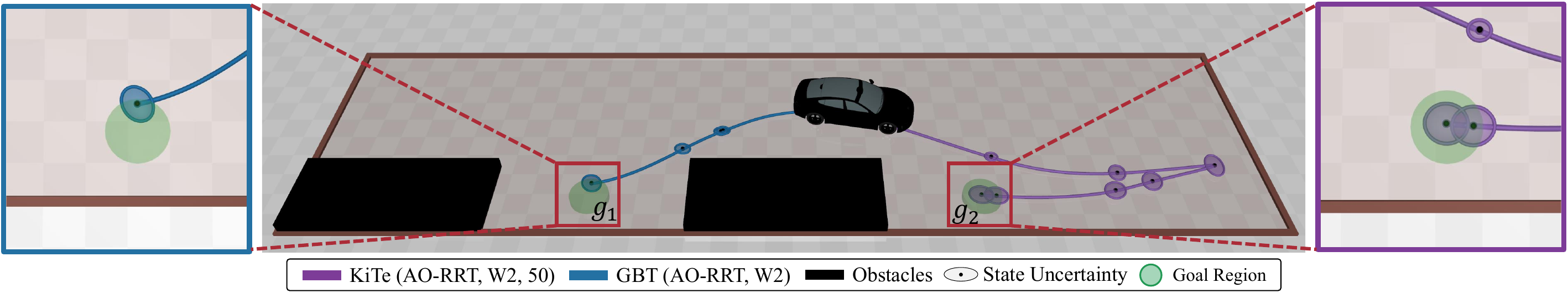}
    \caption{Representative plans in the Car Parking task. GBT tends to choose the nearest feasible parking region and often stops refining once the goal constraint is satisfied. In contrast, \method selects the more desirable goal region and continues improving the terminal belief after feasibility is achieved.}
    \label{fig:car_demo}
\vspace{-3ex}
\end{figure*}

As shown in \autoref{fig:car_result}~(a)~and~(b), AO-RRT-based methods without the terminal cost generally achieve a lower running cost.
Once the terminal cost is added, minimizing running cost is no longer the only objective.
Instead, \method trades part of the path-length objective for a more favorable terminal state with lower terminal cost.
Since the terminal cost is defined with respect to the front parking region, \method strongly prefers $g_2$.
After 30 seconds of planning, \method selects the desired goal region $g_2$ in 99.5\% of its plans, whereas the other methods do so in 68.5\% of the plans.
For methods without the terminal cost, the terminal cost tends to increase over planning time, because they increasingly favor a geometrically closer goal region, which is not necessarily $g_2$.

The task success rate is shown in \autoref{fig:car_result}~(c). Success is defined as the planner returning a plan whose execution reaches any goal regions without collision.
Among all methods, \method achieves the highest realized success rates.
The breakdown in \autoref{fig:car_result}~(d) suggests two complementary effects.
First, planning in belief space improves collision avoidance by enabling chance constraints.
Second, the terminal cost term improves the goal-reaching probability, which is consistent with the theory in \autoref{sec:aorrt-terminal}.
Interestingly, even the L2 surrogate as terminal cost improves success in practice, although it is not a valid theoretical lower bound.

Representative plans are shown in \autoref{fig:car_demo}.
GBT tends to favor the closest feasible parking region because that minimizes the running cost, and it usually stops refining once the terminal belief satisfies the goal-reaching threshold.
In contrast, \method chooses the more desirable goal region and continues improving the terminal belief after finding feasible plans.

Alternatively, one might simply select a preferred goal region and enforce a stricter goal-reaching threshold to achieve goal desirability and high success rate. 
However, it directly makes the feasible planning problem more difficult under the same planning budget.
In contrast, \method preserves the original feasible planning problem and allows the planner to progressively improve goal-reaching probability.

\begin{figure}[!t]
    \centering
    \includegraphics[width=\linewidth]{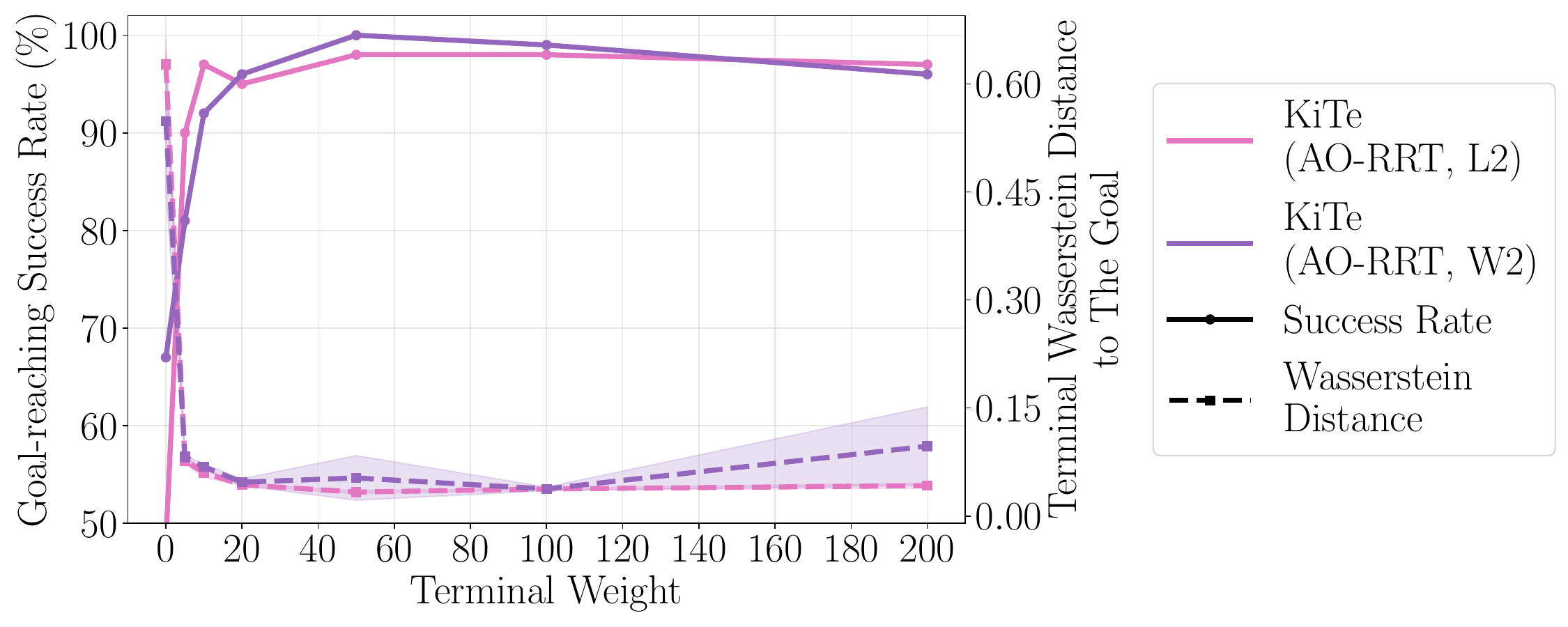}
    \caption{Effect of the terminal weight $w_g$ in the Car Parking task. Intermediate weights provide the best trade-off between empirical goal-reaching success and search efficiency.}
    \label{fig:car_terminal}
\vspace{-2ex}
\end{figure}

We also study the effect of the terminal weight $w_g$.
Specifically, we test
$w_g \in \{0, 5, 10, 20, 50, 100, 200\}$,
where $w_g = 0$ corresponds to planning without a terminal term.
For each value, we plan for 30 seconds and record the terminal Wasserstein distance to the goal, the corresponding theoretical lower bound, and the realized goal-reaching success rate after execution.

As shown in \autoref{fig:car_terminal}, increasing $w_g$ generally reduces the terminal cost and improves the empirical success rate, and the improvement gradually converges.
Theoretically, when $w_g$ becomes large, the cost bound in the state-cost space becomes dominated by the terminal term, which can weaken the pruning effect and reduce the search efficiency.
In practice, the best trade-off is achieved by setting the terminal cost weight in an intermediate range so that the terminal cost is on a scale comparable to the running cost.
Additionally, using the geometric distance to the goal (L2) as the terminal cost yields similar performance, even though it is a simplified surrogate.

This experiment shows that terminal cost plays two important roles in belief space.
It enables the planner to favor a desirable goal region.
Also, it improves the quality of the terminal belief, thereby increasing the goal-reaching success rate.

\subsection{Planar Pushing Task}
\label{sec:experiment-pushing}

Finally, we consider a contact-rich planar pushing task~\cite{activepusher}, in which both the action dynamics and the associated uncertainty are difficult to derive analytically.
We evaluate this task in both MuJoCo simulation and in the real world using objects from the YCB dataset~\cite{ycb}.
In simulation, the robot pushes a Mustard Bottle and a lying Chef Can.
In real-world experiments, we use a Cracker Box and a toy Trash Truck.
The Chef Can and the Trash Truck are intentionally challenging because pushes from certain directions can induce rolling, which leads to complex dynamics and strongly action-dependent uncertainty.

\begin{figure}[!t]
    \centering
    \includegraphics[width=\linewidth]{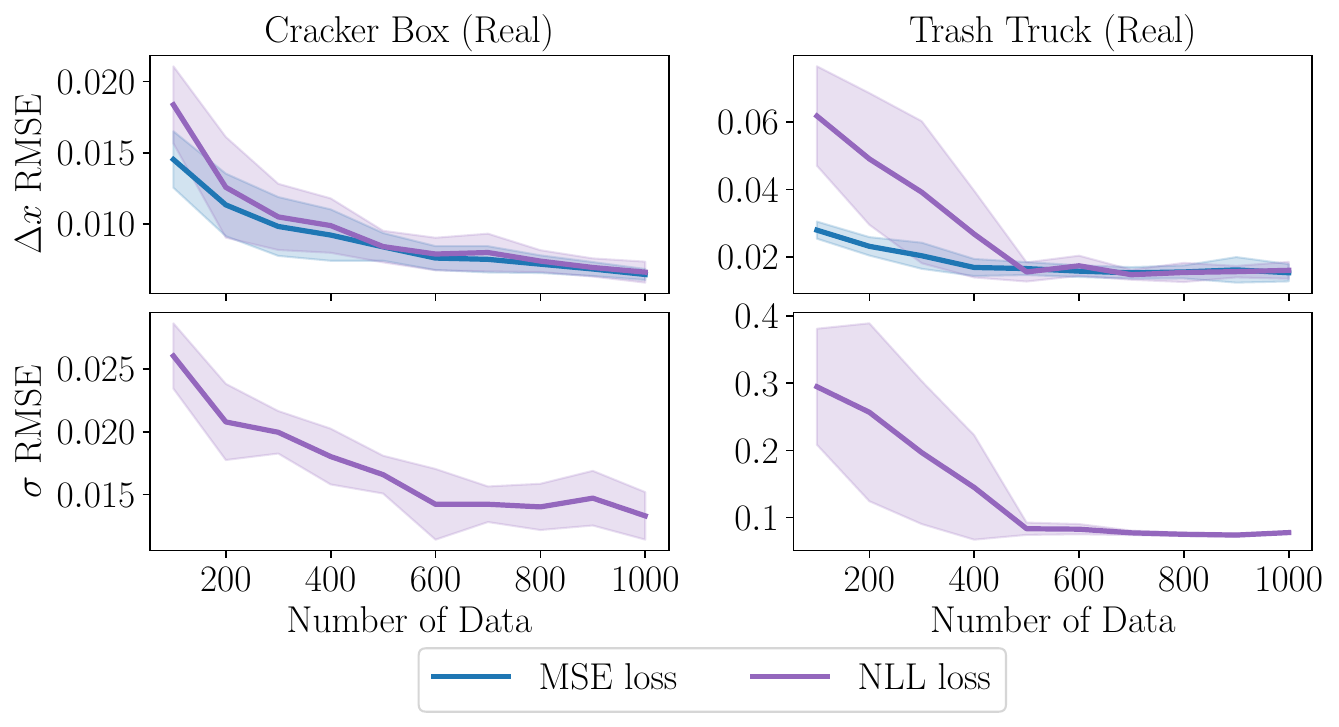}
    \caption{
    Real-world model learning results for the Cracker Box (left) and Trash Truck (right). NLL initially has larger prediction errors but improves uncertainty estimates and recovers prediction accuracy with more data.
    }
    \label{fig:push_learn}
\vspace{-2ex}
\end{figure}
\begin{table*}[!ht]
\centering
\caption{Results of Pushing Experiments}
\label{tab:result}
\begin{tabular}{c|ccc|ccc|ccc|ccc}
\hline
 &
  \multicolumn{3}{c|}{Mustard Bottle (Sim.)} &
  \multicolumn{3}{c|}{Chef Can (Sim.)} &
  \multicolumn{3}{c|}{Cracker Box (Real)} &
  \multicolumn{3}{c}{Trash Truck (Real)} \\ \hline
 &
  \multicolumn{1}{c|}{$\;\,$SR$\;\,$} &
  \multicolumn{1}{c|}{$\;\cost_w\;$} &
  $\cost_\psi$ &
  \multicolumn{1}{c|}{$\;\,$SR$\;\,$} &
  \multicolumn{1}{c|}{$\cost_w$} &
  $\cost_\psi$ &
  \multicolumn{1}{c|}{$\;\,$SR$\;\,$} &
  \multicolumn{1}{c|}{$\cost_w$} &
  $\cost_\psi$ &
  \multicolumn{1}{c|}{$\;\,$SR$\;\,$} &
  \multicolumn{1}{c|}{$\cost_w$} &
  $\cost_\psi$ \\ \hline
Base (SST, L2) &
  \multicolumn{1}{c|}{32} &
  \multicolumn{1}{c|}{0.50} &
  1.46 &
  \multicolumn{1}{c|}{39} &
  \multicolumn{1}{c|}{$\textbf{0.52}$} &
  1.25 &
  \multicolumn{1}{c|}{55} &
  \multicolumn{1}{c|}{0.46} &
  0.93 &
  \multicolumn{1}{c|}{30} &
  \multicolumn{1}{c|}{0.52} &
  1.41 \\ \hline
Base (AO-RRT, L2) &
  \multicolumn{1}{c|}{19} &
  \multicolumn{1}{c|}{$\textbf{0.47}$} &
  1.44 &
  \multicolumn{1}{c|}{36} &
  \multicolumn{1}{c|}{$\textbf{0.52}$} &
  1.27 &
  \multicolumn{1}{c|}{55} &
  \multicolumn{1}{c|}{$\textbf{0.45}$} &
  1.05 &
  \multicolumn{1}{c|}{30} &
  \multicolumn{1}{c|}{$\textbf{0.50}$} &
  1.43 \\ \hline
GBT (SST, W2) &
  \multicolumn{1}{c|}{38} &
  \multicolumn{1}{c|}{0.55} &
  0.98 &
  \multicolumn{1}{c|}{46} &
  \multicolumn{1}{c|}{0.53} &
  1.04 &
  \multicolumn{1}{c|}{60} &
  \multicolumn{1}{c|}{0.47} &
  0.93 &
  \multicolumn{1}{c|}{45} &
  \multicolumn{1}{c|}{0.55} &
  1.01 \\ \hline
GBT (AO-RRT, W2) &
  \multicolumn{1}{c|}{32} &
  \multicolumn{1}{c|}{0.51} &
  1.06 &
  \multicolumn{1}{c|}{35} &
  \multicolumn{1}{c|}{0.53} &
  1.07 &
  \multicolumn{1}{c|}{65} &
  \multicolumn{1}{c|}{0.47} &
  1.01 &
  \multicolumn{1}{c|}{45} &
  \multicolumn{1}{c|}{0.54} &
  1.05 \\ \hline
AP (SST, L2+MU) &
  \multicolumn{1}{c|}{31} &
  \multicolumn{1}{c|}{0.52} &
  1.15 &
  \multicolumn{1}{c|}{46} &
  \multicolumn{1}{c|}{0.54} &
  1.16 &
  \multicolumn{1}{c|}{60} &
  \multicolumn{1}{c|}{0.46} &
  0.92 &
  \multicolumn{1}{c|}{40} &
  \multicolumn{1}{c|}{0.54} &
  1.30 \\ \hline
AP (AO-RRT, L2+MU) &
  \multicolumn{1}{c|}{28} &
  \multicolumn{1}{c|}{0.48} &
  1.19 &
  \multicolumn{1}{c|}{36} &
  \multicolumn{1}{c|}{0.54} &
  1.20 &
  \multicolumn{1}{c|}{60} &
  \multicolumn{1}{c|}{0.46} &
  0.99 &
  \multicolumn{1}{c|}{30} &
  \multicolumn{1}{c|}{0.51} &
  1.27 \\ \hline
$\method$ (AO-RRT, L2, 20) &
  \multicolumn{1}{c|}{57} &
  \multicolumn{1}{c|}{0.57} &
  0.78 &
  \multicolumn{1}{c|}{73} &
  \multicolumn{1}{c|}{0.61} &
  0.72 &
  \multicolumn{1}{c|}{90} &
  \multicolumn{1}{c|}{0.55} &
  $\textbf{0.41}$ &
  \multicolumn{1}{c|}{60} &
  \multicolumn{1}{c|}{0.62} &
  1.01 \\ \hline
$\method$ (AO-RRT, W2, 20) &
  \multicolumn{1}{c|}{$\textbf{86}$} &
  \multicolumn{1}{c|}{0.61} &
  $\textbf{0.48}$ &
  \multicolumn{1}{c|}{$\textbf{88}$} &
  \multicolumn{1}{c|}{0.63} &
  $\textbf{0.56}$ &
  \multicolumn{1}{c|}{$\textbf{95}$} &
  \multicolumn{1}{c|}{0.55} &
  $\textbf{0.41}$ &
  \multicolumn{1}{c|}{$\textbf{85}$} &
  \multicolumn{1}{c|}{0.64} &
  $\textbf{0.59}$ \\ \hline
\end{tabular}
\end{table*}
\begin{figure*}[!t]
\vspace{-2ex}
    \centering
    \includegraphics[width=\linewidth]{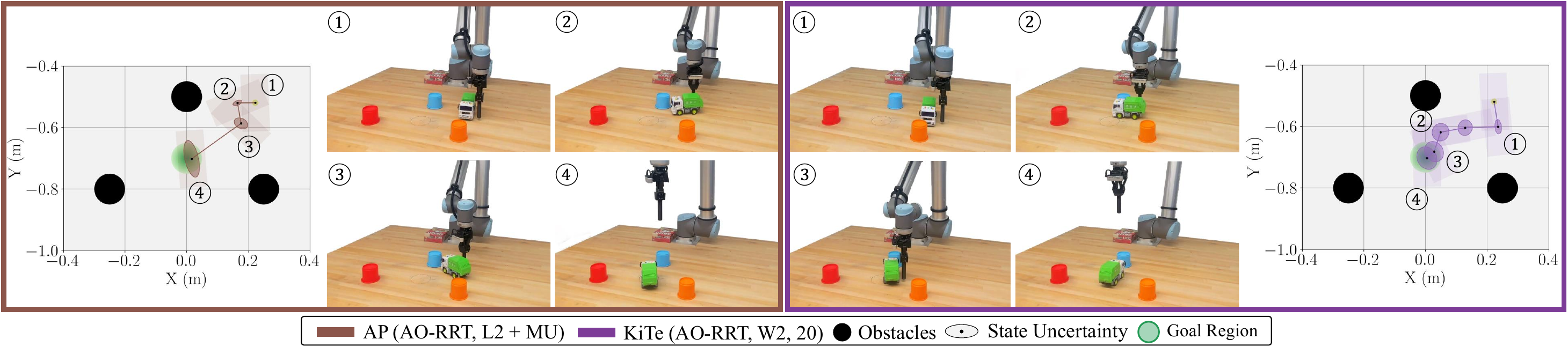}
    \caption{Representative plans for the Trash Truck pushing in the real world. AP tends to favor a geometrically shorter path that may induce larger terminal distance to the goal. In contrast, \method chooses a longer geometric route but reduces terminal uncertainty and distance to the goal.}
    \label{fig:push_demo}
\vspace{-2ex}
\end{figure*}

The object state lies in $\SETwo$ and is represented as
$x=[p_x, p_y, \theta]^\top$.
Each pushing action is represented by a 3-dimensional control $u=[u_s, u_o, u_d]^\top$, where $u_s$ denotes the pushing side, $u_o$ the lateral offset along the contacted side, and $u_d$ the pushing distance.
A global inverse kinematics resolution~\cite{expansion_grr} is used to generate feasible joint-space paths given the sequence of control actions.
Since the task evolves in $\SETwo$, we use the learned Lie-group transition and covariance models in \autoref{eq:learned-mean-lie} and \autoref{eq:learned-cov-lie}.

We adopt an object-centric isotropy assumption, where the effect of a push is invariant to the
object’s current pose. To learn the model, the robot repeatedly executes random pushes and collects up to 1000 action-outcome pairs $\{(u_i,\Delta x_i)\}$ as training data for each object.
We additionally sample 100 different actions and repeat each action 10 times.
The empirical mean and covariance of the repeated outcomes are treated as the ground-truth transition and process uncertainty for evaluation.
We train the model using the NLL loss in \autoref{eq:nll-loss}.
The network is a multi-layer perceptron with four hidden layers of sizes [32, 64, 64, 32] in the shared feature encoder and two layers of sizes [32, 3] in both the mean and the uncertainty heads.

\autoref{fig:push_learn} shows two representative learning results with real-world data, from the Cracker Box and the Trash Truck, where each training is repeated five times with different random seeds.
We evaluate transition prediction, $\Delta x$ and $\sigma$, with root mean square error (RMSE) on the repeated-action test set, and compare NLL loss with standard mean squared error (MSE) loss.
NLL typically requires more data than MSE to reach comparable transition prediction accuracy, particularly for objects with more complex dynamics, because it must jointly learn the mean and covariance.
This additional uncertainty estimation makes the learning problem harder, since inaccurate covariance estimates can distort the weighting of prediction residuals in the likelihood objective.
Once more data are available, NLL improves its covariance estimates and recovers mean-prediction accuracy.
In the planning experiments, we use the models trained with 1000 samples.

The objective of this task is to push the object into the target region, specified by both position and orientation, without colliding with the obstacles.
The task environment is fixed with three cylinder obstacles placed on the table, as shown in \autoref{fig:push_demo}.
Similar to the Car Parking task, regular-space planning uses the $\SETwo$ distance in \autoref{eq:lie-l2-approx}, while belief-space planning uses the approximated Wasserstein distance in \autoref{eq:lie-w2-approx}.
The running cost is measured by the accumulated path length and the terminal cost is the distance to the goal.
For each object, we randomly generate 20 planning problems with different initial object poses and run each planner 5 times per problem in simulation and 1 time in the real world.
The maximum planning time is 10 seconds.
Each returned plan is then executed either in MuJoCo or with the real robot.

We compare against classical baselines Base (SST, L2) and Base (AO-RRT, L2), belief-space baselines GBT (SST, W2) and GBT (AO-RRT, W2) with $p_\text{free}=0.95$, and additionally, the active planning method~\cite{activepusher}.
The latter uses model uncertainty (MU) to bias action sampling with an AO planner, rather than explicitly propagating learned action uncertainty in belief space.
We denote it as AP (SST, L2 + MU) and AP (AO-RRT, L2 + MU).
For the proposed method, we report both \method (AO-RRT, L2, $w_g$) and \method (AO-RRT, W2, $w_g$) with a terminal weight $w_g=20$.

The success rate (SR) of each method, together with its average running cost $\cost_w$ and terminal cost $\cost_\psi$, is summarized in \autoref{tab:result}.
Similar to the Car Parking task, we unify cost measuring with Wasserstein distance by post-processing every plan for comparison.

\method achieves the best terminal quality and the highest execution success rate across the tested objects.
Meanwhile, similar to previous tasks, \method often sacrifices some running-cost optimality, accepting a longer path in exchange for lower terminal Wasserstein distance and higher execution success.
GBT generally outperforms other baselines since it can also explicitly account for uncertainty and enforce safer obstacle avoidance.
However, GBT still optimizes only a running cost and, therefore, cannot continue refining the terminal belief.
By contrast, \method directly optimizes terminal quality and thus further improves goal-reaching success.
Although not reported in \autoref{tab:result}, we observe that AP performs best when training data are limited, e.g., fewer than 400 samples.
In such cases, sampling with model uncertainty helps avoid unreliable actions before an accurate model is available (\autoref{fig:push_learn}).
With sufficient data, belief-space planning becomes more effective as the learned uncertainty improves.
Representative real-world examples for the Trash Truck are shown in \autoref{fig:push_demo}.
In this example, AP tends to favor geometrically short paths that may lead to larger terminal uncertainty.
In contrast, \method (AO-RRT, W2, 20) chooses a longer geometric path but accumulates less terminal uncertainty by planning in belief space and favoring lower-uncertainty pushes.
Leveraging the terminal cost, the resulting terminal belief is also closer to the center of the goal region.
Together with the ability to enforce chance constraints on obstacle avoidance, \method (AO-RRT, W2, 20) achieves the highest success rate among all methods.

Overall, this experiment shows that both dynamics and process uncertainty can be learned directly from interaction data and integrated into planning.
When the learned uncertainty is used together with a terminal cost, the planner can favor low-uncertainty actions, improve terminal quality, and achieve higher goal-reaching success.

\section{Limitations and Future Work}
\label{sec:conclusion}

We conclude by discussing limitations of the proposed method and outlining potential future directions.
First, the current formulation considers only open-loop propagation. Incorporating observation updates and closed-loop feedback policies would more directly connect planning to execution-time robustness, but would require additional modeling and algorithmic extensions.
Second, introducing a terminal cost adds design degrees of freedom such as the terminal weight. While the ablation study in \autoref{sec:experiments} provides guidance for choosing it, a user still must select weights to reflect task preferences and the desired balance.
Third, the learned uncertainty model is diagonal, i.e., variance-only. Although this is simple and effective, diagonal covariances may miss important correlated errors, which can lead to inaccurate belief evolution.

Future work naturally follows these limitations.
The first direction is closing the loop: integrating observation updates or execution-time feedback to adapt the planned belief trajectory during execution~\cite{belief_tree, tamp_real}.
Because the framework involves multiple objectives, a second direction is to reduce manual weight selection by adopting multi-objective formulations that explicitly characterize the trade-off, e.g., via Pareto front~\cite{multiobjective_planning}.
A third direction is richer uncertainty learning. We can learn correlated covariances or hybrid residual-physics models that incorporate structured priors more deeply~\cite{activepusher}.
Finally, we plan to evaluate scalability on more complex, higher-dimensional kinodynamic systems, including tasks with longer horizons, tighter dynamic constraints, and increased state dimensionality.

\appendices
\section{Proof of AO-RRT Lemmas}
\label{app:aorrt-terminal-proof}

\subsection{Proof of \autoref{lem:select}}
\label{app:aorrt-lem-select-proof}


\begin{figure}
    \centering
    \includegraphics[height=0.15\textheight]{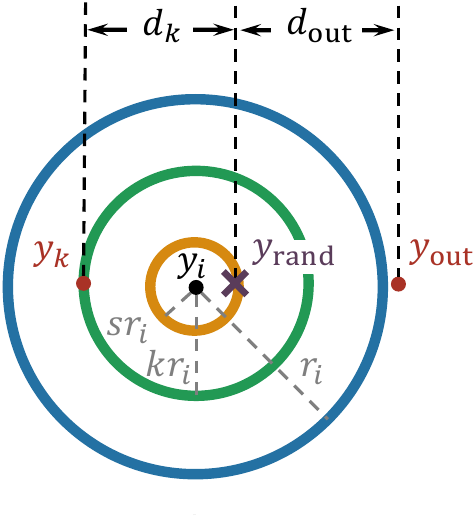}
    \caption{Illustration for proof of \autoref{lem:select}.}
    \label{fig:app_lemma1}
\vspace{-2ex}
\end{figure}

\begin{proof}
%
As shown in \autoref{fig:app_lemma1}, let $y_k \in B_{kr_i}(y_i)$ be a vertex in the tree, and let $y_{\text{out}}$ be any vertex outside the ball $B_{r_i}(y_i)$. 
Consider uniform sampling of a random point $y_{\text{rand}}$ within another ball $B_{sr_i}(y_i)$ where $s = (1 - k)/\, 2$. We have:
\begin{equation}
\begin{gathered}
    d_k = \| y_{\text{rand}} - y_k \| \le (s + k) \,r_i = (1 + k) \,r_i /\, 2, \\
    d_\text{out} = \| y_{\text{rand}} - y_{out} \| > (1 - s) \,r_i = (1 + k) \,r_i /\, 2.
\end{gathered}
\end{equation}
Hence, any point outside $B_{r_i}(y_i)$ is strictly farther from $y_{\text{rand}}$ than $y_k$ is. Therefore, no outside vertex can be selected as the nearest neighbor when the random sample lies in $B_{sr_i}(y_i)$. The only remaining case that $y_k$ is not selected is when there is other nodes in $B_{r_i}(y_i)$ that is closer than $y_k$, which is also desired. The probability of sampling inside $B_{sr_i}(y_i)$ is simply:
\begin{equation}
    \Pr \left(y_{\text{rand}} \in B_{sr_i}(y_i) \right) \ge 
    \frac{\vol{ B_{sr_i}(y_i) }}{\vol{ \Y }} > 0,
\end{equation}
and for every such sample the nearest neighbor must lie inside $B_{r_i}(y_i)$. This yields the stated probability bound.

Note that this bound is conservative. The event $y_{\text{rand}} \in B_{sr_i}(y_i)$ is a sufficient but not necessary condition for the nearest neighbor to lie inside $B_{r_i}(y_i)$. Indeed, random samples outside $B_{sr_i}(y_i)$ may also have nearest neighbor in $B_{r_i}(y_i)$. Nevertheless, this bound is sufficient for our analysis to guarantee that selecting a vertex in $B_{r_i}(y_i)$ is strictly positive.
\end{proof}

\subsection{Proof of \autoref{lem:propagate}}
\label{app:aorrt-lem-propagate-proof}

\begin{figure}[!h]
    \centering
    \includegraphics[height=0.15\textheight]{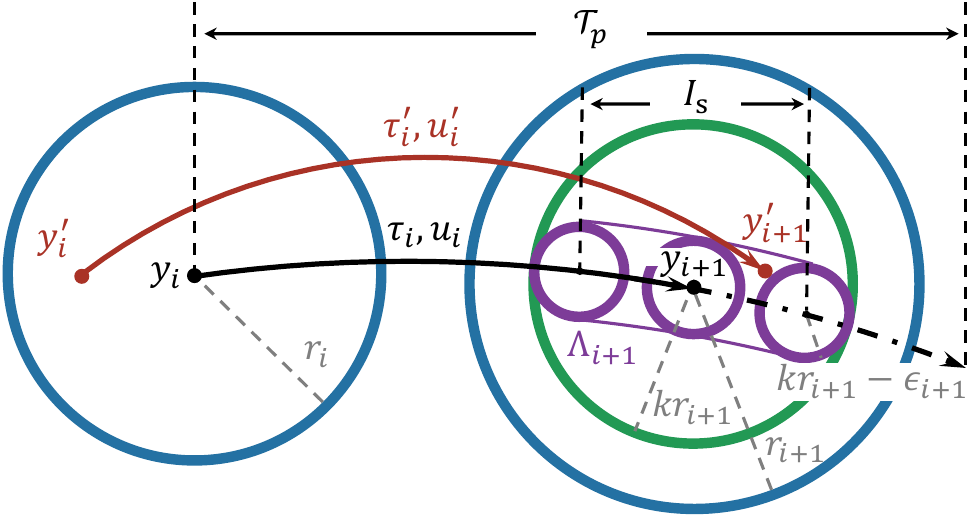}
    \caption{Illustration for proof of \autoref{lem:propagate}.}
    \label{fig:app_lemma2}
\vspace{-2ex}
\end{figure}

\begin{proof}
As illustrated in \autoref{fig:app_lemma2}, let $y(t)$ denote the trajectory segment starting from $y_i$ under constant control $u_i$, where $t \in [0,\T_p]$ is the local time along this segment with $t=0$ at $y_{i}$ and $t=\tau_i$ at $y_{i+1}$. For $t > \tau_i$, the segment is conceptually extended to $\T_p$ (black curve). Let $y'(t)$ denote the perturbed trajectory from a perturbed initial state $y_{i}' \in B_{r_{i}}(y_{i})$ under a perturbed constant control $u_i' \in B_{\Delta u_i}(u_i)$ (red curve).

Define the augmented error:
\begin{equation}
e(t) \;=\; y(t) - y'(t).
\end{equation}
Taking the derivative of $e(t)$ gives:
\begin{equation}
\begin{aligned}
\dot{e}(t)
&= \dot{y}(t) - \dot{y}'(t) \\
&= \aug(y(t), u_i) - \aug(y'(t), u_i')
\qquad \text{(\autoref{eq:aug_dynamics})}
\\
&= \aug(y(t), u_i) - \aug(y'(t), u_i) 
+ \aug(y'(t), u_i) - \aug(y'(t), u_i')
\end{aligned}
\end{equation}
By Lipschitz continuity of $\aug$ (\autoref{eq:lip_z}), for all $t$:
\begin{equation}
\begin{gathered}
\|\dot{e}(t)\|
\le K_y \|e(t)\| + K_u \|u_i - u_i'\|.
\end{gathered}
\end{equation}
Applying Gr\"{o}nwall’s inequality on $[0, \T_p]$ yields:
\begin{equation}
\label{eq:e(t)}
\begin{aligned}
\|e(t)\|
&\le
e^{K_y t} \|e(0)\|
+
\tfrac{K_u}{K_y} (e^{K_y t} - 1)\|u_i - u_i'\|
\\
&\le
e^{K_y \T_p} \, r_i
+
\tfrac{K_u}{K_y} \,(e^{K_y \T_p} - 1)\Delta u_i,
\end{aligned}
\end{equation}
where the last inequality holds because the initial state remains inside $B_{r_i}(y_0)$ and $e^{K_y t}$ is nondecreasing in $(0, \T_p]$.

Pick $\epsilon_{i+1}$ such that $0 < \epsilon_{i+1} < k r_{i+1}$, and for each $t \in [0,\T_p]$, consider the ball $B_{k r_{i+1} - \epsilon_{i+1}}(y(t))$ (the purple balls in \autoref{fig:app_lemma2}). Let $\Lambda_{i+1}$ denote the union of all such balls that satisfies $B_{k r_{i+1} - \epsilon_{i+1}}(y(t)) \subset B_{k r_{i+1}} (y_{i+1})$. This set forms a “tube” around the segment with nonzero time interval $I_s$ (the purple region in \autoref{fig:app_lemma2}). Hence, $\Lambda_{i+1}
\subset B_{k r_{i+1}}(y_{i+1})$.

To ensure that the perturbed trajectory enters $\Lambda_{i+1}$ at some time $\tau_i' \in I_s$,
it suffices to guarantee:
\begin{equation}
\|e(\tau_i')\| \le k r_{i+1} - \epsilon_{i+1}.
\end{equation}
Using the bound in \autoref{eq:e(t)}, it is enough to impose
\begin{equation}
e^{K_y \T_p} r_i
+
\tfrac{K_u}{K_y} \,(e^{K_y \T_p} - 1) \Delta u_i
\le
k r_{i+1} - \epsilon_{i+1}.
\end{equation}
Substituting $r_{i+1} = r_i (e^{K_y\T_p} + \beta)/k$ gives:
\begin{equation}
e^{K_y \T_p} r_i
+
\tfrac{K_u}{K_y} \,(e^{K_y \T_p} - 1) \Delta u_i
\le
r_i(e^{K_y\T_p} + \beta) - \epsilon_{i+1},
\end{equation}
which is equivalent to
\begin{equation}
\Delta u_i
\le
\frac{\beta r_i - \epsilon_{i+1}}
     {\frac{K_u}{K_y} \,(e^{K_y \T_p} - 1)}.
\end{equation}
Choosing $\epsilon_{i+1}$ such that $0 < \epsilon_{i+1} < \beta r_i$
ensures that the right-hand side is positive.
Consequently, any perturbed input $u_i'$ satisfying 
$\lvert u_i' - u_i \rvert < \Delta u_i$ will satisfy this inequality.
Also note that:
\begin{equation}
\epsilon_{i+1}
< \beta r_i
< r_i(e^{K_y\T_p} + \beta)
= k r_{i+1},
\end{equation}
so, the earlier requirement $\epsilon_{i+1} < k r_{i+1}$ is also
satisfied.

In AO-RRT, the duration $\tau_i'$ is sampled uniformly from $(0,\T_p]$ and the control $u_i'$ is sampled uniformly from $\U$. Hence, the probability of sampling $\tau_i' \in I_s$ and
$u_i' \in B_{\Delta u_i}(u_i)$ is:
\begin{equation}
p_\tau
=
\frac{\vol{I_s}}{\vol{(0,\T_p]}}
> 0,
\quad
p_u
=
\frac{\vol{B_{\Delta u_i}(u_i)}}{\vol{\U}}
> 0.
\end{equation}
For every such sample, the propagated state $y_{i+1}'$ lies in $B_{k r_{i+1}}(y_{i+1})$:
\begin{equation}
\Pr \left( y_{i+1}' \in B_{k r_{i+1}}(y_{i+1}) \right)
\ge
p_\tau \cdot p_u
> 0.
\end{equation}

Finally, we remark that this lower bound is conservative. Different choices of $\epsilon_{i+1}$ lead to different $I_s$ and $\Delta u_i$, and the true probability corresponds to integrating all such admissible choices.
Nonetheless, the bound is sufficient to show that the probability of reaching the next ball is strictly positive.
\end{proof}

\subsection{Proof of \autoref{lem:terminal}}
\label{app:aorrt-lem-terminal-proof}

\begin{figure}[!h]
    \centering
    \includegraphics[height=0.15\textheight]{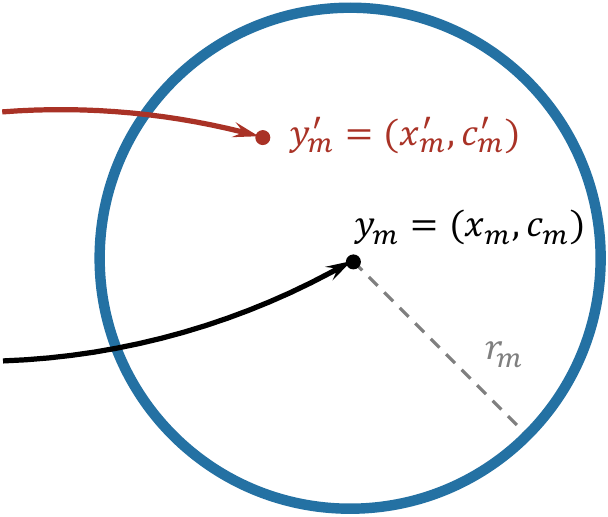}
    \caption{Illustration for proof of \autoref{lem:terminal}.}
    \label{fig:app_lemma3}
\vspace{-2ex}
\end{figure}

\begin{proof}
Recall the distance in $\Y$ is defined by the L2 norm (\autoref{eq:dist_aug}). Thus, As shown in \autoref{fig:app_lemma3}, for any $y_m' \in B_{r_m}(y_m)$:
\begin{equation}
\begin{gathered}
| c_m' - c_m | \le r_m, \quad
\| x_m' - x_m \| \le r_m.
\end{gathered}
\end{equation}
By \autoref{as:lip-terminal}, the terminal cost $\terminal(\cdot)$ is Lipschitz continuous with constant $K_\terminal$.
Therefore, the total cost satisfies
\begin{equation}
\begin{aligned}
c_T' 
&= c_m' + \terminal_m' 
\\
&\le (c_m + | c_m' - c_m |) + (\terminal_m + | \terminal_m' - \terminal_m |) 
\\
&\le (c_m + | c_m' - c_m |) + (\terminal_m + K_\terminal \| x_m' - x_m \|)
\quad \text{(\autoref{eq:lip_terminal})}
\\
&\le (c_m + r_m) + (\terminal_m + K_\terminal r_m) 
\\
&= c_T + (1 + K_\terminal) r_m 
\\
&\le c_T + (1 + K_\terminal) \tfrac{\varepsilon \, c_T}{1 + K_\terminal}, 
\quad \text{(definition in~\autoref{eq:lem3_r})}
\\
&= (1 + \varepsilon)\, c_T,
\end{aligned}
\end{equation}
This concludes the proof.
\end{proof}

\section{Proof of Belief-space Dynamics and Cost Functions}
\label{app:aorrt-belief-proof}

\subsection{Proof of \autoref{pr:belief-lip-dyn}}
\label{app:aorrt-belief-lip-dyn-proof}

\begin{proof}
Let $\gamma^* \in \Gamma(b,b')$ be an optimal coupling, and let $(X,X') \sim \gamma^*$. Then $X \sim b$, $X' \sim b'$, and $W_2^2(b,b') = \mathbb{E}[D_\X(X,X')^2]$ (\autoref{eq:wasserstein}).
For each realization $(X,X') = (x,x')$, let $\gamma_\eta^*(x,x') \in \Gamma\bigl(\mathcal{N}\big(0,Q(x, u)\big), \mathcal{N}\big(0,Q(x', u)\big)\bigr)$ be an optimal coupling for the noise, and let $(H, H') \mid (X=x, X'=x') \sim \gamma_\eta^*(x,x')$.
Then the conditional marginals are
$H \sim \mathcal{N}\big(0, Q(x,u)\big)$ and $H' \sim \mathcal{N}\big(0, Q(x',u)\big)$.

Define $X^+ = F(X,u)\diamond H$ and $X_x'^+ = F(X',u)\diamond H'$. Then $X^+ \sim F_b(b,u)$ and $X_x'^+ \sim F_b(b',u)$.
Although $(X,X')$ is drawn from an optimal coupling of $b$ and $b'$, the propagated pair $(X^+,X_x'^+)$ may not remain an optimal coupling of $F_b(b,u)$ and $F_b(b',u)$. From the definition of Wasserstein distance (\autoref{eq:wasserstein}):
\begin{equation}
\label{eq:prop1-1}
\begin{aligned}
&\quad\,\, W_2^2(F_b(b,u), F_b(b',u))
\\
&\le 
\mathbb{E}[D_\X(X^+, X_x'^+)^2]
\\
&= \mathbb{E}[D_\X\big(F(X,u)\diamond H,\; F(X',u)\diamond H'\big)^2].
\end{aligned}
\end{equation}
With the triangle inequality of $D_\X(\cdot,\cdot)$:
\begin{equation}
\label{eq:prop1-2}
\begin{aligned}
&\quad\,\, W_2^2(F_b(b,u), F_b(b',u))
\\
&\le \mathbb{E}\!\bigl[ \bigl(D_\X\big(F(X,u)\diamond H,\; F(X',u)\diamond H\big)
\\
&\quad\,\, + D_\X\big(F(X',u)\diamond H,\; F(X',u)\diamond H'\big)\bigr)^2 \bigr]
\\
&\le \mathbb{E}\!\bigl[\bigl(K_x^{F\eta}D_\X(X,X') 
\\
&\quad\,\, + K_\eta^{F\eta}D_\X(H,H')\bigr)^2\bigr]
\quad \text{(\autoref{as:lip-injection})}
\\
&\le 2(K_x^{F\eta})^2 \mathbb{E}[D_\X(X,X')^2]
+ 2(K_\eta^{F\eta})^2 \mathbb{E}[D_\X(H,H')^2].
\end{aligned}
\end{equation}
Now, conditioning on $(X,X')$, the coupling $(H,H')$ is optimal for $\mathcal{N}\big(0,Q(X,u)\big)$ and $\mathcal{N}\big(0, Q(X',u)\big)$, thus
\begin{equation}
\begin{aligned}
\mathbb{E}[D_\X(H,H')^2]
&= \mathbb{E}\!\left[\mathbb{E}[D_\X(H,H')^2 \mid X,X']\right]
\\
&= \mathbb{E}\!\left[W_2^2\big(\mathcal{N}\big(0, Q(X,u)\big), \mathcal{N}\big(0, Q(X',u)\big)\big)\right]
\\
&\le (K_x^Q)^2 \mathbb{E}[D_\X(X,X')^2].
\quad \text{(\autoref{as:lip-noise})}
\end{aligned}
\end{equation}
Substituting this into the previous bound gives
\begin{equation}
\begin{aligned}
&\quad\,\, W_2^2(F_b(b,u), F_b(b',u))
\\
&\le \left( 2(K_x^{F\eta})^2 + 2(K_\eta^{F\eta}K_x^Q)^2 \right) \mathbb{E}[D_\X(X,X')^2]
\\
&= \left( 2(K_x^{F\eta})^2 + 2(K_\eta^{F\eta}K_x^Q)^2 \right) W_2^2(b,b').
\end{aligned}
\end{equation}
Taking square roots yields the first Lipschitz constant: 
$K_x^{F_b} = \left( 2(K_x^{F\eta})^2 + 2(K_\eta^{F\eta}K_x^Q)^2 \right)^{\frac12}$.

For the control argument, let $X \sim b$. For each realization $X=x$, let
$\tilde{\gamma}_\eta^*(x) \in \Gamma\bigl(\mathcal{N}\big(0,Q(x,u)\big), \mathcal{N}\big(0,Q(x,u')\big)\bigr)$
be an optimal coupling for the noise, and let $(H,H') \mid (X=x) \sim \tilde{\gamma}_\eta^*(x)$.
Then, conditioned on $X=x$, the marginals satisfy
$H \sim \mathcal{N}\big(0, Q(x,u)\big)$ and $H' \sim \mathcal{N}\big(0, Q(x,u')\big)$.

Define $X^+ = F(X,u)\diamond H$ and $X_u'^+ = F(X,u')\diamond H'$. Then $X^+ \sim F_b(b,u)$ and $X_u'^+ \sim F_b(b,u')$, and $(X^+,X_u'^+)$ is a valid coupling of $F_b(b,u)$ and $F_b(b,u')$. Similarly:
\begin{equation}
\begin{aligned}
&\quad\,\, W_2^2(F_b(b,u), F_b(b,u'))
\\
&\le \mathbb{E}[D_\X(X^+, X_u'^+)^2]
\quad \text{(definition of \autoref{eq:wasserstein})}
\\
&= \mathbb{E}[D_\X\big(F(X,u)\diamond H,\; F(X,u')\diamond H'\big)^2]
\\
&\le \mathbb{E}\!\bigl[\bigl(D_\X\big(F(X,u)\diamond H,\; F(X,u')\diamond H\big)
\;\; \text{(triangle inequality)}
\\
&\quad\,\, + D_\X\big(F(X,u')\diamond H,\; F(X,u')\diamond H'\big)\bigr)^2\bigr]
\\
&\le \mathbb{E}\!\left[\bigl(K_u^{F\eta}\|u-u'\| + K_\eta^{F\eta}D_\X(H,H')\bigr)^2\right]
\;\; \text{(\autoref{as:lip-injection})}
\\
&\le 2(K_u^{F\eta})^2\|u-u'\|^2
+ 2(K_\eta^{F\eta})^2 \mathbb{E}[D_\X(H,H')^2].
\end{aligned}
\end{equation}
Again, conditioning on $X$, the coupling $(H,H')$ is optimal for $\mathcal{N}\big(0, Q(X,u)\big)$ and $\mathcal{N}\big(0, Q(X,u')\big)$, so
\begin{equation}
\begin{aligned}
\mathbb{E}[D_\X(H,H')^2]
&= \mathbb{E}\!\left[\mathbb{E}[D_\X(H,H')^2 \mid X]\right]
\\
&= \mathbb{E}\!\left[W_2^2\bigl(\mathcal{N}\big(0, Q(X,u)\big), \mathcal{N}\big(0, Q(X,u')\big)\bigr)\right]
\\
&\le (K_u^Q)^2 \|u-u'\|^2.
\quad \text{(\autoref{as:lip-noise})}
\end{aligned}
\end{equation}
Therefore,
\begin{equation}
\begin{aligned}
&\quad\,\, W_2^2(F_b(b,u), F_b(b,u'))
\\
&\le \left( 2(K_u^{F\eta})^2 + 2(K_\eta^{F\eta}K_u^Q)^2 \right)\|u-u'\|^2.
\end{aligned}
\end{equation}
Taking square roots yields the second Lipschitz constant with
$K_u^{F_b} = \left( 2(K_u^{F\eta})^2 + 2(K_\eta^{F\eta}K_u^Q)^2 \right)^{\frac12}$.

Thus, the belief dynamics $F_b$ is Lipschitz continuous in both arguments with respect to the Wasserstein distance.
\end{proof}

\subsection{Proof of \autoref{pr:belief-lip-cost}}
\label{app:aorrt-belief-lip-cost-proof}

\begin{proof}
Since the Wasserstein distance is a valid metric in \B~\cite{belief-distance}, it satisfies the reverse triangle inequality:
\begin{equation}
\label{eq:reverse_tri_ineq}
\begin{gathered}
|W_2(b, b'') - W_2(b', b'')| \le W_2(b, b'),
\end{gathered}
\end{equation}
where $b, b', b'' \in \B$. 
If the running cost is the Wasserstein distance between the current belief and the next belief, then:
\begin{equation}
\begin{aligned}
&\quad\,\,| \lengthb \big(b, F_b(b, u)\big) - \lengthb \big(b', F_b(b', u)\big) |
\\
&= | W_2\big(b, F_b(b, u)\big) - W_2\big(b', F_b(b', u)\big) |
\\
&\le | W_2\big(b, F_b(b, u)\big) - W_2\big(b', F_b(b, u)\big) |
\\
&\quad+ | W_2\big(b', F_b(b, u)\big) - W_2\big(b', F_b(b', u)\big) |
\\
&\le W_2(b, b') + W_2\big(F_b(b, u), F_b(b',u)\big) 
\quad \text{(\autoref{eq:reverse_tri_ineq})}
\\
&\le (1 + K_x^{F_b}) W_2(b, b').
\quad \text{(\autoref{pr:belief-lip-dyn})}
\end{aligned}
\end{equation}
Also for the control argument,
\begin{equation}
\begin{aligned}
&\quad\,\, | \lengthb(b, F_b(b, u)) - \lengthb(b, F_b(b, u')) |
\\
&= | W_2\big(b, F_b(b, u)\big) - W_2\big(b, F_b(b, u')\big) |
\\
&\le W_2\big(F_b(b, u), F_b(b, u')\big)
\quad \text{(\autoref{eq:reverse_tri_ineq})}
\\
&\le K_u^{F_b} \|u - u'\|.
\quad \text{(\autoref{pr:belief-lip-dyn})}
\end{aligned}
\end{equation}
Given a constant belief $b_c$, by \autoref{eq:reverse_tri_ineq}, the terminal cost with Wasserstein distance satisfies:
\begin{equation}
| \terminalb(b) - \terminalb(b') |
= | W_2(b, b_c) - W_2(b', b_c) |
\le W_2(b, b').
\end{equation}
In summary, both cost functions are Lipschitz continuous with constants $K_x^\lengthb=1 + K_x^{F_b}$,  $K_u^\lengthb=K_u^{F_b}$ and $K_x^\terminalb = 1$.
\end{proof}

\bibliographystyle{IEEEtran_ShortURL}
\bibliography{references}

\end{document}